\documentclass[11pt, a4paper, oneside]{amsart}
\usepackage[margin=1.15in, top=1.1in, bottom=1.2in, headsep=0.4in, footskip=0.55in]{geometry}
\usepackage{amsaddr}

\usepackage{array}
\newcolumntype{P}[1]{>{\centering\arraybackslash}p{#1}}

\usepackage{verbatim}
\usepackage{dsfont}

\usepackage{setspace}

\usepackage{microtype}
\usepackage{graphicx}
\usepackage{subcaption}
\usepackage{booktabs} % for professional tables

\usepackage{amsmath}
\usepackage{amsfonts}
\usepackage{amssymb}
\usepackage{amsthm}
\usepackage{pifont}
\usepackage[utf8]{inputenc} % allow utf-8 input
\usepackage{underscore} % not having to escape underscore

\newcommand{\cmark}{\ding{51}}%
\newcommand{\xmark}{\ding{55}}%

\DeclareMathOperator*{\argmin}{arg\,min}

\usepackage[shortlabels]{enumitem}

\usepackage{parskip}
\usepackage{natbib}
\usepackage{mathtools}
\usepackage{times}
\usepackage{float}

\usepackage[hyphens]{url}
\usepackage[breaklinks=false]{hyperref}
\hypersetup{breaklinks=true}

\restylefloat{table}

\newcommand{\E}{\mathbb{E}}
\newcommand{\R}{\mathbb{R}}
\newcommand{\C}{\mathbb{C}}
\newcommand{\p}{\mathbb{P}}

\newcommand{\Ocal}{\mathcal{O}}
\newcommand{\Hcal}{\mathcal{H}}
\newcommand{\Xcal}{\mathcal{X}}

\newcommand{\Acal}{\mathcal{A}}

\newcommand{\Htsr}{\Hcal \otimes \Hcal}
\newcommand{\mat}[1]{\mathbf{\mathrm{#1}}}
\newcommand{\krn}{\mat{\Phi}}
\newcommand{\one}{\mathds{1}}
\newcommand{\hs}{\mathrm{HS}(\mathcal{H})}
\newcommand{\tr}{\mathrm{Tr}}

\newtheorem{theorem}{Theorem}

\newtheorem{proposition}{Proposition}
\newtheorem{lemma}{Lemma}
\theoremstyle{definition}

\usepackage{changepage}

\usepackage[T1]{fontenc}
\usepackage{kantlipsum}
\usepackage[usenames,dvipsnames]{xcolor}
\usepackage[breakable, theorems, skins]{tcolorbox}
\tcbset{enhanced}

\DeclareRobustCommand{\mybox}[2][gray!20]{%
\begin{tcolorbox}[   
        breakable,
        left=12pt,
        right=0pt,
        top=3pt,
        bottom=2pt,
        colback=#1,
        colframe=#1,
        width=\dimexpr\textwidth\relax,
        enlarge left by=0mm,
        boxsep=5pt,
        arc=0pt,outer arc=0pt,
        ]
        #2
\end{tcolorbox}
}

\usepackage{fancyhdr}
\usepackage{wrapfig}

\begin{document}

\title[Nyström kernel PCA]{Kernel PCA with the Nyström method}
\author{Fredrik Hallgren}
\address{Department of Statistical Science \\
	 %Faculty of Mathematical and Physical Sciences \\
         University College London \\}
\email[F. Hallgren]{fredrik.hallgren@ucl.ac.uk}
\thanks{\emph{Acknowledgements.} The author is hugely grateful to John Shawe-Taylor, Dino Sejdinovic, Paul Northrop and Michael Arbel for invaluable comments and feedback.}

\setlength{\abovedisplayskip}{16pt}
\setlength{\belowdisplayskip}{20pt}

\maketitle

\begin{abstract}
The Nyström method is one of the most popular techniques for improving the scalability of kernel methods. However, it has not yet been derived for kernel PCA in line with classical PCA. In this paper we derive kernel PCA with the Nyström method, thereby providing one of the few available options to make kernel PCA scalable. We further study its statistical accuracy through a finite-sample confidence bound on the empirical reconstruction error compared to the full method. The behaviours of the method and bound are illustrated through computer experiments on multiple real-world datasets. As an application of the method we present kernel principal component regression with the Nyström method, as an alternative to Nyström kernel ridge regression for efficient regularized regression with kernels.

%...where good performance and expected behaviour are demonstrated
\vspace{12pt}

\hspace{-12pt}\textbf{Keywords:} Kernel methods, non-parametric statistics, confidence interval, dimensionality reduction, unsupervised learning, learning theory, PCA, functional PCA, PCR, MDS
\end{abstract}

\thispagestyle{empty}

\vspace{32pt}
\begin{spacing}{0}
\begin{small}
\tableofcontents
\end{small}
\end{spacing}

\newpage
%\emph{}
%\vspace{-12pt}
\section{Introduction}

\vspace{12pt}

% TODO maybe remove the second sentence to make it more concise and read better
Kernel methods generalize classical statistical methods to discover non-linear patterns in data \citep{hofmann2008kernel}. They have been demonstrated to achieve excellent performance in many application domains and it is straightforward to apply them to non-numeric data, such as graphs or text \citep{vishwanathan2010graph, lodhi2002text}. Through a near arbitrary non-linear mapping of data points into a Hilbert space they offer remarkable flexibility whilst providing a precise mathematical framework for statistical analyses. A host of linear statistical methods have been adapted to be used with kernels, including Fisher discriminant analysis (FDA) \citep{mika1999fisher}, independent component analysis (ICA) \citep{bach2002kernel}, instrumental variable (IV) regression \citep{singh2019kernel}, and many more. Kernel PCA is a non-linear version of principal component analysis (PCA), a ubiquitous method to discover the most important directions of variation in data \citep{pearson1901principal}. PCA may be used for dimensionality reduction, exploratory data analysis, anomaly detection, discriminant analysis, clustering, or as a general preprocessing step for regression or classification \citep{jolliffe2002principal, wold1987principal}.

% TODO maybe mention that the Nystrom method has been widely applied, e.g. in MDS, transformers, randomized linear algebra, etc etc " and its applicability extends way beyond the realm of kernel methods, to include randomized linear algebra, MDS and transfomer architecture in deep learning
The other side of the coin of kernel methods is their large computational requirements, as they generally scale in the number of data points rather than the number of data dimensions. As a remedy, various approximations have been proposed, such as the Nyström method, which randomly selects a smaller subset of data points and looks for solutions in their linear span. The Nyström method also plays an important role in recent state-of-the-art implementations of kernel methods \citep{rudi2017falkon, ma2017diving, meanti2020kernel, carratino2021park, frangella2021randomized}.

The need for approximate methods becomes particularly acute for kernel PCA, since it relies on the eigendecomposition of the kernel matrix, which requires about $9n^3 + \Ocal(n^2)$ floating-point operations, as opposed to $\frac{1}{3}n^3 + \Ocal(n^2)$ floating-point operations for the solution of a linear system by way of the Cholesky decomposition when performing regression \citep[Chapters~4, 8]{golub2013matrix}. Despite this fact, kernel PCA with the Nyström method has not yet been derived in line with linear PCA, significant previous research interest notwithstanding. %...several previous partial attempts notwithstanding. % ...maybe write "contributions" instead of attempts, or remove the last phrase. 

In this paper we derive kernel PCA with the Nyström method, providing orthonormal principal components in the span of the Nyström subset that maximize the variance of the data, without assuming that the data has zero mean, as well as the associated principal scores\footnote{Different conventions exist for the terminology of PCA. Throughout this paper we will take the \emph{principal components} to mean the vectors defining the subspaces that maximize the variance of the data i.e. the eigenvectors of the centred covariance operator or matrix. These are elsewhere sometimes referred to as the \emph{principal axes}.}. The principal scores are \begin{wrapfigure}{r}{0.44\textwidth}
    \vspace{-13pt}
    \hspace{-6pt}
    \includegraphics[width=0.45\textwidth]{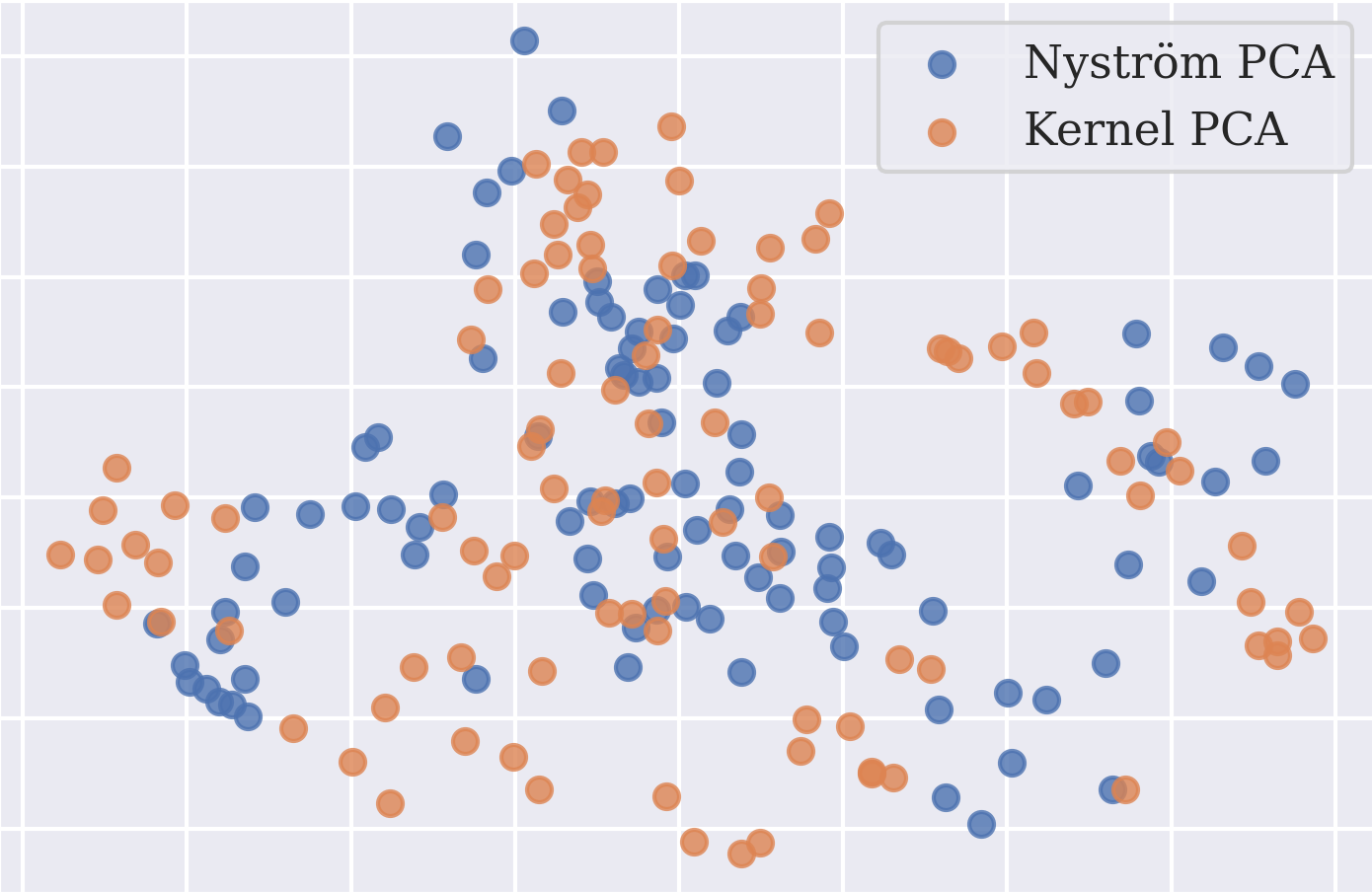}
    \vspace{-30pt}
\end{wrapfigure}
perhaps of particular interest, since they allow for the method to be used as a preprocessing step before applying supervised learning methods, by virtue of providing a new representation of data points in the new coordinate system defined by the principal components. The figure to the right shows the first two dimensions for these representations for an example with a dataset of images of handwritten digits, in comparison with standard full kernel PCA, where the Nyström subset uses only a tenth of the full dataset\footnote{Please see {\SMALL \url{https://github.com/fredhallgren/nystrompca}} for details}.

The principal scores are given as follows. First let $K_{mm}$ be $m$ randomly subsampled rows and columns of the original kernel matrix $K$, and $K_{nm}$ be the same $m$ subsampled columns. Centring the data in feature space corresponds to adjusting these matrices through
\begin{align*}
    K'_{nm} &= K_{nm} - \one_n K_{nm}  - \widetilde{K} \, \one_n^{n, m} + \one_n \widetilde{K} \, \one_n^{n, m} \\[0.5em]
    K'_{mm} &= K_{mm} - \one_n^{m, n} K_{nm} - K_{mn}\one_n^{n, m}  + \one_n^{m, n}\widetilde{K} \, \one_n^{n, m}
\end{align*}
%where $\widetilde{K} = K_{nm} K_{mm}^{-1} K_{mn}$ and $\one_n$, $\one_n^{n, m}$ and $\one_n^{m, n}$ are $n \times n$, $n \times m$ and $m \times n$ matrices respectively with all elements equal to $\frac{1}{n}$. Each element of $\one_n K_{nm}$ or $\one_n^{m, n} K_{nm}$ equals the mean of the values in $K_{nm}$ in that element's column, and each element of $\widetilde{K} \one_n^{n, m}$ or $K_{mn} \one_n^{n, m}$ equals the mean across that row. The matrices $\one_n^{m, n} \widetilde{K} \one_n^{n, m}$ and $\one_n \widetilde{K} \one_n^{n,m}$ are constant with each element equal to the sum of all elements of $\widetilde{K}$ divided by $n^2$. Now create an approximate kernel matrix through
where $\widetilde{K} = K_{nm} K_{mm}^{-1} K_{mn}$ and $\one_n$, $\one_n^{n, m}$ and $\one_n^{m, n}$ are $n \times n$, $n \times m$ and $m \times n$ matrices respectively with all elements equal to $\frac{1}{n}$. Now create an approximate kernel matrix through
\[
    \widetilde{K}' = \frac{1}{n}K_{mm}^{\prime \,-1/2} K'_{mn} K'_{nm} K_{mm}^{\prime \, -1/2}
\]
and calculate its eigenvalues $\widetilde{\lambda}_j$ and eigendecomposition $V\widetilde{\Lambda}V^T$, where $M^{-1/2} = U D^{-1/2} U^T$ for a matrix $M$ with eigendecomposition $UDU^T$. The scores are then given by $W = K'_{nm} K_{mm}^{\prime \, -1/2} V$, which is a new data matrix with observations along the rows, and the variances of the new data variables (in the columns) are given by $\widetilde{\lambda}_j$. The method has time complexity $\Ocal(nm^2)$ which is the same as when the Nyström method is applied to regression. Full details may be found in Section~\ref{sec:kpca} on page~\pageref{sec:kpca}.

The method centres the data in the feature space, as is the case for linear PCA as well as the original presentation of kernel PCA \citep{pearson1901principal, jolliffe2016principal, scholkopf1998nonlinear}. Without this adjustment, the perpendicular lines defined by the principal components are forced to go through the origin, no longer minimizing the reconstruction error in an unconstrained fashion and requiring an assumption of zero-mean data in feature space. If the assumption does not hold then the results may be wildly different from the true values. The assumption may be especially contrived for kernel methods, since for any positive kernel function, like the popular radial basis functions or Cauchy kernels, it will \emph{never hold}, since the corresponding feature maps $\phi(x_i) = k(x_i, \cdot)$ in the reproducing kernel Hilbert space are positive everywhere. Subtracting the mean from the input data will not give zero-mean data in feature space. Sometimes the first eigenvalue of the uncentred kernel matrix will account for the mean being different from zero, and the subsequent eigenvalues approximately correspond to the full set of eigenvalues of the centred kernel matrix, but this is not always the case and the correspondence is not exact. % TODO maybe skip everything from "Subtracting..." or just this sentence

We further study the statistical accuracy of the proposed method. In the special case when the number of subsampled data points for the Nyström method equals the PCA dimension, then both the empirical and true reconstruction errors of the Nyström method equal the corresponding reconstruction errors for kernel PCA constructed directly using only the subset of data points. For the general case we provide a finite-sample confidence bound (a confidence interval) with $\Ocal(m^3)$ time complexity that doesn't require that we have observed the entire dataset, only the subset of data \citep{ramachandran2020mathematical}. In line with essentially all existing results on the accuracy of standard kernel PCA we here assume that data has zero mean. The result states that with high confidence, the difference between the empirical reconstruction errors of Nyström kernel PCA and full kernel PCA is less than or equal to a data-dependent quantity
\[
    R_n(\widetilde{V}_d) - R_n(\hat{V}_d)  \le h\left(\sup_x k(x, x), \; \left\{\hat{\lambda}_j\right\}_{j=1}^{d+1}, \; m, \; n \right)
\]
which depends on the maximum value of the kernel function, the eigenvalues of the kernel matrix from the subset of randomly subsampled data points, the size of this subset $m$ and the total size of the dataset $n$, where $h(\cdot)$ is a fixed function. All quantities in the bound are known or can be calculated from the data at hand. Please see Section \ref{sec:accuracy} on page~\pageref{sec:accuracy} for the complete result.

We illustrate and evaluate the proposed method and derived confidence bound through experimental analysis using several different datasets and kernel functions. We first compare the accuracy of Nyström kernel PCA with a number of other unsupervised learning methods, where its performance is seen to be very close to full kernel PCA, whilst being much more efficient. Then we illustrate the behaviour of the bound across different PCA dimensions. The source code for all the experiments is publicly available at {\small\url{https://github.com/fredhallgren/nystrompca}}. 
%... "The method also outperforms kernel PCA with random features, the only other methods propsed to make kernel PCA more efficient without assuming zero-mean data".
%...The package includes a command-line tool to easily rerun the experiments for different parameter values... TODO maybe mention the command-line tool here or that it can be installed with the python package manager.

%An efficient C++ implementation of the method is also available as part of the \emph{Shogun Machine Learning Toolbox}\footnote{\url{https://www.shogun.ml}}, a popular open source machine learning package which is accessible from multiple different programming languages.

From the proof of the confidence bound one can deduce sharper versions of several previous concentration results based on the covariance operator, including from \cite{rosasco2010learning}, \cite{de2005learning} and \cite{blanchard2019concentration}. Please see Section \ref{sec:cor} on page~\pageref{sec:cor} for details. %Other techniques used in the proof may also be of independent interest.

To demonstrate the use of Nyström kernel PCA with supervised learning methods we apply it to the regression problem to present kernel principal component regression with the Nyström method. Principal component regression (PCR) performs a linear regression on the principal scores from the top principal components instead of the original data and introduces regularization for improved generalization. We also illustrate the method through experimental analysis and compare it to kernel ridge regression with the Nyström method. In summary, the prediction for a data point $x^*$ is given by
\[
    \hat{y} =
    \bar{y} +
    y^{\prime \, T} K'_{nm} K^{\prime \, -1/2}_{mm}V_d \widetilde{\Lambda}^{-1}_d V_d^T K^{\prime \, -1/2}_{mm} \, \widetilde{\kappa}(x^*)
\]
where $y' = (y_1 - \bar{y}, \, y_2 - \bar{y}, \, ..., \, y_n - \bar{y})^T$ and $\widetilde{\kappa}(x) = \kappa_m(x) - K_{mn} \mathbf{1}_n - \one_n^{m, n} K_{nm} K_{mm}^{-1}  \kappa_m(x)   + \one_n^{m, n}\widetilde{K} \mathbf{1}_n$ with $\mathbf{1}_n$ a length-$n$ column vector given by $\mathbf{1}_n = (\frac{1}{n}, \frac{1}{n}, \; ..., \; \frac{1}{n})^T$ and $\kappa_m(x) = (k(x_1, x), \; k(x_2, x), \; ..., \; k(x_m, x))^T$. Using similar techniques we also present a novel derivation of standard kernel PCR with centred data in feature space, where a prediction is given by
\[
    \hat{y} = \bar{y} +  y^{\prime \, T} Q_d \Lambda_d^{-1} Q_d^T \kappa'(x^*)
\]
where $Q_d \Lambda_dQ_d^T$ is the truncated eigendecomposition of $K' = K - \one_nK - K \one_n + \one_nK\one_n$ and $\kappa'(x) = \kappa(x) - \one_n \kappa(x) - K \mathbf{1}_n + \one_n K \mathbf{1}_n$ with $\kappa(x) = (k(x_1, x), k(x_2, x), ..., k(x_n, x))^T$. We refer to Section~\ref{sec:appl} on page~\pageref{sec:appl} for the full derivation and experimental results.

A summary of our main contributions is as follows
\begin{enumerate}[(1)]
    \item Deriving kernel PCA with the Nyström method
    \item A result on the accuracy in the special case of $d = m$ for both the empirical and true errors
    \item A finite-sample confidence bound for the empirical error in the general case
    \item Presenting kernel principal component regression with the Nyström method
    \item Novel specification of kernel PCR with centred regressors
    \item Sharper versions of some concentration results from previous literature
\end{enumerate}

In the next section we give an overview of previous work (Section \ref{sec:prevwork}), then go through relevant background (Section \ref{sec:background}), present the main method (Section \ref{sec:kpca}), study the special case when $d=m$ (Section \ref{sec:prelude}), provide the confidence bound on the accuracy of the method (Section \ref{sec:accuracy}), conduct experimental analysis of the method and bound (Section \ref{sec:exp}), present kernel principal component regression with the Nyström method (Section \ref{sec:appl}) and finally conclude with a summary and outlook (Section \ref{sec:conc}). Proofs are in the appendix.

\vspace{18pt}

\subsection*{Notation} Upper-case letters will be used for matrices and operators and generally for random variables, unless they represent data points before they are observed. Vectors in $\R^p$ will be denoted by small letters and parameters fitted to data often by letters from the Greek alphabet. A row vector $v$ in $\R^p$ with elements $v_1, v_2, ..., v_p$ will be written $(v_1, v_2, ..., v_p)$ and its $i$th element will also be written $(v)_i$. The transpose of a vector or matrix is $v^T$. If not stated otherwise all Euclidean vectors will be column vectors. The arithmetic mean of a vector is denoted $\bar{v}$. Indices for data points will be denoted by $i$, $r$, or  $\ell$; indices for eigenvectors or dimensions will be denoted by $j$, $k$, $p$ or $q$. Estimated quantities will often be denoted by $\hat{\cdot} \,$, approximations by $\, \widetilde{\cdot} \,$ and centred quantities by $\cdot \,'$. Empirical quantities may be superscripted or subscripted by the number of observations used in the estimate. The probability density function of a measure $\p_Y$ will be denoted by $p_Y(y)$. The symbol $Y$ will be used for a generic random variable; the symbols $T$ or $L$ for a generic operator and the symbol $M$ for a generic matrix. The linear span of a set of vectors $A$ is written $\mathrm{span}\{A\}$ or $\left\langle A \right\rangle$. The cardinality of a basis for the space $V$ is written $\dim(V)$. The real part of a complex number is denoted $\mathrm{Re}\{x + iy\} = x$.

The symbol $\Ocal(\cdot)$ denotes Big-O notation \citep{sipser2012introduction}. The function $\lambda_j(\cdot)$ returns the $j$th eigenvalue, in decreasing order, of its argument, and the symbol $\lambda_{<d}$ denotes the sum of the largest $d$ eigenvalues $\lambda_1, \lambda_2, ..., \lambda_d$. If $v$ is majorized by $u$ we write $v \succ u$. The symbol $:=$ denotes the introduction of new notation, i.e. $a := b$ means that $b$ will be denoted by $a$, and vice versa for $a =: b$. The binary operators $\vee$ and $\wedge$ are defined as $a \vee b = \max\{a,b\}$ and $a \wedge b = \min\{a,b\}$. %The notation $\otimes^n v$ is used for $v \otimes v \otimes \cdots \otimes v$ for $n$ instances of the symbol $v$.

The functional $\| \cdot \|$ denotes the operator norm or the Euclidean norm, depending on the context. For other norms the space will always be specified. For an operator $T$, we let $T^*$ denote its adjoint. The image of an operator is $\mathrm{Im}(T)$ and its null space (also called its kernel) is $\mathrm{Ker}(T)$.

\vspace{32pt}

\section{Previous work} \label{sec:prevwork}

\vspace{12pt}

The study of the statistical accuracy of kernel PCA, or of the related problems of functional PCA \citep{besse1986principal, hall2006properties2} and PCA of a Hilbert space-valued random variable \citep{besse1991approximation}, was initiated in \cite{dauxois1982asymptotic}. They demonstrated the consistency of the reconstruction error and asymptotic normality of the empirical reconstruction error and principal components about the true quantities. The asymptotics of kernel PCA was also studied in \cite{koltchinskii2000random}. A concentration inequality for the empirical reconstruction error versus its expectation, using McDiarmid's inequality \citep{mcdiarmid1989method}, was provided in \cite{shawe2002eigenspectrum} and the same authors later presented a confidence bound on the expected empirical reconstruction error versus the true error \citep{shawe2005eigenspectrum}, which is based on Rademacher complexities \citep{bartlett2002rademacher}. In this bound the expectation is with respect to the data point to be projected and the confidence with respect to different training datasets. A similar bound, as well as a version for centred kernel PCA, was later presented in \cite{blanchard2007statistical}. Approximate confidence bounds for both the principal values and components were given in \cite{hall2006properties} based on the bootstrap method \citep{davison1997bootstrap}. However, these results are not immediately applicable to kernel PCA since the kernel is defined on a compact subset of $\R \times \R$. Error bounds for the principal components were also provided in \citep{zwald2005convergence}. The first PAC-Bayes bounds for kernel PCA were recently given in \cite{haddouche2020upper}.

The statistical accuracy of kernel PCA has been widely studied under the assumption of zero-mean data in feature space. To the best of our knowledge, only one bound has been presented in previous literature for centred kernel PCA \citep{blanchard2007statistical}, but this bound is more conservative than the available bounds for uncentred kernel PCA. This is because the bound is based on a bound for uncentred kernel PCA, with an additional term to account for the error introduced by mean-adjustment. However, uncentred kernel PCA minimizes the reconstruction error under the constraint that the subspaces maximizing the variance go through the origin and so the reconstruction error will be much larger.  Consequently there is strictly no bound available for standard kernel PCA without an assumption of zero-mean data in feature space, that is as accurate as the existing bounds for uncentred kernel PCA.
% ...This can even be by a factor $10$ or more, which a comparison of the top eigenvalues of the kernel matrix versus its centred version on some real data can demonstrate.

The Nyström method has been widely studied for different settings and assumptions. Originally developed for the discretization of integral equations \citep{nystrom1930praktische, banach1932theorie}, it was adapted to kernel methods in \cite{williams2001using} and applied to regression. The accuracy of the approximate kernel matrix versus the full kernel matrix, considering the full dataset as fixed, has been studied in a number of papers, please see \cite{gittens2016revisiting} and references therein. The study of the accuracy of the Nyström method as applied to regression culminated in the seminal work by \cite{rudi2015less} as a bound on the expected regression error with general assumptions.

A recent paper \citep{sterge2020gain} applied the Nyström method to kernel PCA, but only explicitly derived the first principal component and also assumed data to have zero mean in feature space. In the current work we derive all the principal components without assuming that data has zero mean, and in addition we derive the principal scores and provide an empirical evaluation of our method. They also presented a probabilistic inequality for the true reconstruction error with respect to the empirical subspace, which depends on the maximum value of the kernel function $\sup_x k(x,x)$, the total number of data points $n$ and the covariance operator $C$ from the unknown true distribution $\p$. The main difference to our confidence bound is that we bound the empirical reconstruction error (the estimate) and only in terms of known quantities. As a corollary they also presented an asymptotic rate of convergence under the assumption of polynomial decay of the eigenvalues of $C$. 

Even more recently \cite{sterge2021statistical} presented a similar analysis to \cite{sterge2020gain}, but with one way of centring the data. This method is different from the one considered here, with the top principal component given by
\[
    \tilde{\phi}_1 = \sqrt{m} G_m^* K_{mm}^{-1/2} u_1
\]
where $u_1$ is the top eigenvector of $\frac{1}{n-1}K_{mm}^{-1/2} (K_{mn} - K_{mn} \one_n) K_{nm} K_{mm}^{-1/2}$ and $G^*_m$ is given by $ \alpha \mapsto \frac{1}{\sqrt{m}}\sum_{k=1}^m \alpha_k k(x_k, x)$. This method does not minimize the reconstruction error or maximize the variance, and it does not recover standard kernel PCA when $m=n$.

If we assume data to have zero mean in feature space, then our derived principal scores are somewhat similar to the \emph{virtual samples} of \cite{golts2016linearized} which were introduced in the context of dictionary learning \citep{aharon2006k}. These are obtained through the projection of the full dataset onto the Nyström subset and may also be used as a drop-in replacement for the original data points.

Another related method was described in \cite{iosifidis2016nystrom} . They derive low-rank data representations from the Nyström kernel matrix, similar to the \emph{virtual samples} above, including for novel unseen data points. However, the representations for novel data points require calculation of the top eigenvalues of the full kernel matrix, and so is $\Ocal(n^3)$ in time and leads to no improvement in computational efficiency compared to standard kernel PCA. They also propose a method to centre the data in feature space, although this is done in order to make $\R^m$ into a subspace of $\Hcal$ and the centred representations are different from the principal scores derived below. The centring of the matrices $K_{nm}$ and $K_{mm}$ is the same as the one used here, but these matrices are applied differently.
% Removed (\emph{sic})

The above two methods represent the data points directly through their projections onto the truncated eigenspace of the Nyström subset and do not attempt to find the representations in the span of this subset that best describe the data, unlike the method presented in the current paper, which calculates the representations that have maximum variance for a given number of retained dimensions.

% TODO figure out that their eigenvectors are the correct ones, and exactly how they relate to our principal scores and components. 
In \cite{gisbrecht2015metric} the Nyström method was studied in the context of distance matrices and MDS, which is a related area to kernel PCA (also see \cite{de2004sparse, platt2005fastmap}). They calculated the eigenpairs of $\widetilde{K} = K_{nm}K^{-1}_{mm}K_{mn}$, from which the principal components, scores and values could be derived, although these specific quantities were not provided. They also applied the \emph{double centering} procedure, which is used in MDS to convert a distance matrix into a matrix of inner products. When applied to the Nyström kernel matrix $\widetilde{K}$ this procedure becomes similar to the centring used in the current paper, although it is not quite the same and it is applied for different purposes.

% TODO maybe say that we actually show that the variance is maximized etc... (if makes sense)

In the table below we summarize previous applications of the Nyström method to kernel PCA in terms of important properties when applying the Nyström method or developing variations of PCA.

\medmuskip=2mu
\thinmuskip=1mu
\thickmuskip=3mu

\renewcommand{\arraystretch}{1.5}
\begin{table}[h]
{\Small
\vspace{6pt}
    \caption{Comparison with previous methods of desirable properties of the application of the Nyström method to kernel PCA}
    \label{tab:datasets}
\begin{center}
\begin{tabular}{P{2cm}P{1.2cm}P{1.5cm}P{1.6cm}P{1.3cm}P{1.5cm}P{1.6cm}P{1.5cm}}
    \toprule
  \;\;\;\;\;\;\;\;\;  & Maximizes variance & Without zero-mean assumption & Full set of principal components & Principal scores & Scores for new data& Recovers kernel PCA when $n=m$ & Experimental evaluation  \\
    \midrule
  This paper \newline & \noindent\color{ForestGreen}\cmark   &  \noindent\color{ForestGreen}\cmark &  \noindent\color{ForestGreen}\cmark & \noindent\color{ForestGreen}\cmark & \noindent\color{ForestGreen}\cmark& \noindent\color{ForestGreen}\cmark & \noindent\color{ForestGreen}\cmark \\

Sterge et al (2021) & \noindent\color{red}\xmark  &  \noindent\color{ForestGreen}\cmark  &   \noindent\color{red}\xmark   & \noindent\color{red}\xmark & \noindent\color{red}\xmark& \noindent\color{red}\xmark & \noindent\color{red}\xmark \\

Sterge et al (2020) &      \noindent\color{ForestGreen}\cmark      & \noindent\color{red}\xmark &  \noindent\color{red}\xmark   & \noindent\color{red}\xmark & \noindent\color{red}\xmark& \noindent\color{ForestGreen}\cmark & \noindent\color{red}\xmark \\

Golts and Elad (2016) & \noindent\color{red}\xmark    & \noindent\color{red}\xmark  &  \noindent\color{red}\xmark    & \noindent\color{ForestGreen}\cmark  &  \noindent\color{ForestGreen}\cmark  &  \noindent\color{ForestGreen}\cmark  &  \noindent\color{ForestGreen}\cmark \\

Iosifidis and Gabbouj (2016) &   \noindent\color{red}\xmark        & \noindent\color{ForestGreen}\cmark    &  \noindent\color{red}\xmark    &  \noindent\color{ForestGreen}\cmark & \noindent\color{red}\xmark & \noindent\color{ForestGreen}\cmark & \noindent\color{ForestGreen}\cmark \\
    \bottomrule
\end{tabular}
\end{center}
\vspace{6pt}
}
\end{table}

% Revert to defaults
\medmuskip=4mu
\thinmuskip=3mu
\thickmuskip=5mu

\vspace{12pt}

There are few other methods suggested in the literature to make kernel PCA scalable. Stochastic optimization was applied to kernel PCA in \cite{zhang2016stochastic}, but under the assumption of zero-mean data in feature space. The zero-mean assumption was also employed in \cite{balcan2016communication}, who presented a distributed algorithm for kernel PCA. Random features have also been applied to kernel PCA, corresponding to applying linear PCA to the approximate covariance matrix obtained from the random features, which allows for data to have non-zero mean, but assumes a shift-invariant kernel function \citep{sriperumbudur2017approximate}. Consequently, to the best of our knowledge, other than the Nyström method there is no other procedure proposed in previous literature to make kernel PCA scalable, for arbitrary kernel functions or without the restrictive assumption of zero-mean data. %... which excludes many popular kernels, including the linear one... (?)

\vspace{16pt}

\section{Background} \label{sec:background}

\vspace{12pt}

We have a reproducing kernel Hilbert space $\Hcal$ (RKHS) of functions from a set $\Xcal$ to the real numbers. Associated with each RKHS is a symmetric positive definite kernel $k: \Xcal \times \Xcal \rightarrow \R$ with a reproducing property $\langle k(x, \cdot), f \rangle_{\Hcal} = f(x)$ for which the point evaluation $f \mapsto \langle k(x, \cdot), f \rangle_{\Hcal}$ is bounded. The kernel maps each element $x \in \Xcal$ to an element $\phi(x) := k(x, \cdot) \in \Hcal$. We assume throughout that $\Hcal$ is separable, which will be the case for example if $k$ is continuous and $\Xcal$ is compact \citep{paulsen2016introduction}. 

We have observations $\{x_i\}_{i=1}^n$ of an $\Xcal$-valued random variable $X: (\Omega, \Acal, \p) \rightarrow (\Xcal, \Acal_{\Xcal}, \p_X)$ where $\p_X(A) = \p(X^{-1}(A))$ \citep{cohn1980measure, graham2011simulation}. We obtain a random variable $Z = \phi(X) \in \Hcal$ with observations $z_i = \phi(x_i)$, assuming that $\phi$ is measurable, which is the case for example when $k$ is continuous. We assume $Z$ is absolutely continuous and that it has a continuous density and so all $z_i$ will be distinct. Its expectation in $\Hcal$ is given by $\E[Z] = \int Z d\p$ in the sense of Bochner. We assume $Z$ is in $L^2(\Omega, \Acal, \p; \Hcal)$ with norm $(\E[\|Z\|^2_{\Hcal}])^{1/2} = ( \int \|Z\|^2_{\Hcal}d\p )^{1/2}$ \citep{ledoux2013probability}.

Principal component analysis (PCA) of the zero-mean random variable $Z \in \Hcal$ constructs an optimal subspace $V_d \subset \Hcal$, of dimension $d$, such that the so-called reconstruction error
\[
	R(V) = \E \left[ \|P_VZ - Z\|_{\Hcal}^2 \right]
\]
is minimized, where $P_V: \Hcal \rightarrow \Hcal$ is the projection of (a realization of) $Z$ on a subspace $V$ \citep{besse1991approximation}. This is termed the \emph{true} reconstruction error \citep{blanchard2007statistical}. Since $Z$ is square-integrable the reconstruction error always exists and is finite.

In other words, the optimal $d$-dimensional subspace $V_d$ is given by
\[
    V_d = \argmin_{\dim(V)=d} \E \left[ \|P_V Z - Z \|^2_{\Hcal} \right]
\]
An estimate of the optimal subspace $V_d$ is obtained from the data $\{z_i\}_{i=1}^n$ by minimizing the \emph{empirical} reconstruction error
\[
    R_n(V) = \frac{1}{n}\sum_{i=1}^n \| P_{V} z_i - z_i \|^2_{\Hcal}
\]
which has a unique minimum since all eigenvalues are distinct \citep{blanchard2007statistical}. We denote the estimated subspace by $\hat{V}_d$. One may also consider the true reconstruction error with respect to the empirical subspace, given by
\[
    R(\hat{V}_d) = \E \left[ \|P_{\hat{V}_d}Z - Z \|_{\Hcal}^2 \right]
\]
where the expectation may be taken both with respect to $Z$ and $\hat{V}_d$, or treating the subspace as fixed; as well as the expected value of the empirical reconstruction error, given by

\[
    \E \left[ R_n(\hat{V}_d) \right] = \E \left[ \frac{1}{n}\sum_{i=1}^n \| P_{\hat{V}_d} z_i - z_i \|^2_{\Hcal} \right]
\]

When the random variable $Z$ does not have zero mean, the smallest reconstruction error is obtained from the centred random variable $Z' = Z - \E[Z]$
\[
  R(V_d) = \min_{\dim(V)=d} \E[\|P_V Z' - Z' \|_{\Hcal}^2]
\]
and similarly for the empirical reconstruction error replacing $z_i$ by $z_i' = z_i - \frac{1}{n}\sum_{\ell=1}^n z_{\ell}$.

Alternatively, instead of minimizing the reconstruction error of the centred random variable over $d$-dimensional subspaces $V$, one may minimize over affine subspaces with respect to the original random variable, and also optimize with respect to the term used for centring
\begin{align*}
   R(V_d) = \min_{\substack{a \in \Hcal  \\ \dim(V)=d}} \E[\|P_{a + V} Z - Z \|_{\Hcal}^2]
   =\min_{\substack{a \in \Hcal  \\ \dim(V)=d}} \E[\|P_V (Z - a) - (Z - a) \|_{\Hcal}^2]
\end{align*}
where $a$ is the translation of the vector space $V$, and whose optimal value is known to equal $\E[Z]$, and $P_{a + V}Z = a + P_V(Z - a)$ is the affine projection.

The \emph{covariance operator} is an element $C(u,v) \in \Htsr$ in the tensor product of bilinear functionals on $\Hcal$, given by $C(u,v) = \E [ Z \otimes Z ]$. The \emph{centred} covariance operator is given by
\[
    C'(u,v)
    = \E[ (Z - \E[Z]) \otimes (Z - \E[Z])]
    = \E[Z' \otimes Z']
\]
Identifying $\Htsr$ with the space $\mathrm{HS}(\Hcal)$ of Hilbert-Schmidt operators on $\Hcal$ by way of the mapping of elementary tensors $u \otimes v \mapsto \langle \cdot, u \rangle_{\Hcal} v$ we obtain $C' = \E[ \langle \, \cdot, Z'\rangle_{\Hcal} Z' ]$. When we refer to the covariance operator we may either refer to the tensor in $\Htsr$ or the operator in $\hs$.

A Hilbert-Schmidt operator $L$ is an operator on a Hilbert space $\Hcal$ with finite Hilbert-Schmidt norm, given by $\|L\|_{\hs} = \sum_j \| Le_j \|_{\Hcal}$ for any orthonormal basis $\{e_j\}_j$ in $\Hcal$ \citep[Chapter~5]{davies2007linear}. It is a Hilbert space, with inner product $\langle L_1, L_2 \rangle_{\hs} = \sum_j \langle L_1 e_j, L_2 e_j \rangle_{\Hcal}$. The Hilbert-Schmidt norm is always larger than or equal to the operator norm, $\|L\| \le \|L\|_{\hs}$, and if $\Hcal$ is finite it coincides with the Frobenius norm, $\|L\|_{\hs} = \|M\|_F$ where $M$ is a matrix representation of $L$ \citep{kreyszig1989introductory}. In the following, the use of a matrix representation for a Hilbert-Schmidt operator will often be implicit, and where applicable we pad the matrix representation with trailing zeros.

The covariance operator $C'$ is compact, since it is Hilbert-Schmidt, and so its spectrum is countable and all spectral values are eigenvalues apart from possibly $0$. Since $C'$ is infinite-dimensional, by assumption, the value $0$ is always a spectral value. Furthermore, the covariance operator is self-adjoint, and so the spectrum is real and the resolvent spectrum is empty. Finally, it is positive and so the spectrum is positive.

The sum of the smallest eigenvalues of the centred covariance operator $C'$ equals the minimum true reconstruction error of the centred random variable $Z' = Z - \E[Z]$. The eigenvectors form a countable orthonormal basis of $\mathrm{Im}(C')$, which can be extended to a countable orthonormal basis for the entire space, since $\Hcal$ is separable. Denoting the eigenvalues by $\{\lambda_j\}_{j=1}^{\infty}$ in decreasing order, the minimum reconstruction error can be written $R(V_d) = \sum_{j=d+1}^{\infty} \lambda_j$.

Replacing the measure $\p_Z$ on $\Hcal$ by the empirical measure $\p_n = \frac{1}{n}\sum_{i=1}^n \delta_{z_i}$, where $\delta_x$ is the Dirac delta function, we obtain the empirical covariance operator $C'_n: \Hcal \rightarrow \Hcal$
\[
    C'_n
    = \frac{1}{n} \sum_{i=1}^n \langle \, \cdot, z'_i \, \rangle_{\Hcal} z'_i
\]
We denote its eigenvalues by $\hat{\lambda}^n_1, \hat{\lambda}^n_2, ..., \hat{\lambda}^n_n$ in decreasing order and the corresponding eigenvectors by $\hat{\phi}^n_1, \hat{\phi}^n_2, ..., \hat{\phi}^n_n$. It has finite rank, and so the spectrum only contains eigenvalues, and may or may not include $0$. The minimum empirical reconstruction error is given by its smallest eigenvalues, $R_n(\hat{V}_d) = \sum_{j=d+1}^n \hat{\lambda}^n_j$, and it can be decomposed as $C'_n = \sum_{j=1}^n \hat{\lambda}^n_j \langle \, \cdot \, , \hat{\phi}^n_j \rangle_{\Hcal}\hat{\phi}^n_j$.

If $s: \Xcal \times \Xcal \rightarrow \R$ is square-integrable in the second variable, then the operator given by
\[
    T_s f = \int_{\Xcal} s(x, y) f(y) d\p_X(y)
\]
is an isometry of $L^2(\Xcal, \Acal_{\Xcal}, \p_X; \R)$ into the RKHS with kernel $k(x, y) = \int_{\Xcal} s(x, z) s(z, y) d\p_X(z)$ \citep{paulsen2016introduction}. One may also consider the integral operator
\[
    T_k f = \int_{\Xcal} k(x, y) f(y) d\p_X(y)
\]
 which is equal to $T_k = T^2_s$ and whose eigenvalues equal those of the covariance operator $C$ \citep{shawe2005eigenspectrum}.

If one replaces the probability measure $\p_X$ by its empirical equivalent $\p_n = \frac{1}{n}\sum_{i=1}^n \delta_{x_i}$with respect to the data points $\{x_i\}_{i=1}^n$ one again obtains an empirical operator $T_n$ % on $L^2(\Xcal, \Acal_{\Xcal}, \p_n; \R)$
\[
    T_n f = \int_{\Xcal} k(x, y) f(y) d\p_n(y)
    = \frac{1}{n} \sum_{i=1}^n k(x, x_i)f(x_i)
\]
The sampling operator $G_n$ is defined through $f \mapsto \frac{1}{\sqrt{n}}(f(x_1), f(x_2), ..., f(x_n))$ and it is an isometry of $L^2(\Xcal, \Acal_{\Xcal}, \p_X; \R)$ into $\R^n$ which identifies $T_n$ with $K$ \citep{koltchinskii2000random}. Its adjoint $G^*_n$ is given by $\alpha \mapsto \frac{1}{\sqrt{n}} \sum_{i=1}^n \alpha_i k(x_i, x)$ \citep{rudi2015less}. Furthermore, $C_n = G_n^*G^{}_n$ and $\frac{1}{n}K = G^{}_nG_n^*$.

And so the eigenvalues of the empirical kernel integral operator $T_n$ are the same as the eigenvalues of the kernel matrix, and its eigenvectors are given by \citep{bengio2004learning}
\[
    \hat{\psi}_j^n
    = \frac{\sqrt{n}}{\hat{\lambda}_j^n} \sum_{i=1}^n u_{j,i} k(x_i, x)
    = \frac{\sqrt{n}}{\hat{\lambda}_j^n} u_j^T \kappa(x)
\]
The values of $\hat{\psi}_j^n(x)$ at the points $x_1, x_2, ..., x_n$ equal the corresponding entries in the eigenvector of the kernel matrix $K$, $\hat{\psi}_j^n(x_i) = (u_j)_i$, where $u_j$ is the $j$th eigenvector of $K$.

If we randomly sample $m < n$ indices $S = \{r_1, r_2, ..., r_m\}$ and then take the values of $\hat{\psi}^m_r, \; r \in S$ at all the data points $x_1, x_2, ..., x_n$, and normalize by $\frac{1}{\sqrt{n}}$, we obtain the Nyström approximation \citep{williams2001using},
\begin{align} 
	\widetilde{\lambda}_j &= \frac{n}{m} \hat{\lambda}_j^m \label{eq:nysval} \\[1em] 
        \widetilde{u}_j &= \sqrt {\frac{m}{n}} \frac{1}{\hat{\lambda}_j^m} K_{nm} u_j \label{eq:nysvec}
\end{align}
where $\hat{\lambda}_j^m$ are the eigenvalues of $K_{mm}$, which contains the $m$ subsampled columns and rows of $K$ corresponding to the chosen indices, and $K_{nm}$ is the $m$ subsampled columns. Multiplying together the approximate eigenvalues (\ref{eq:nysval}) and eigenvectors (\ref{eq:nysvec}) one so obtains an approximate kernel matrix $\widetilde{K} = K_{nm} K^{-1}_{mm} K_{mn}$, where $K_{mn}$ is the transpose of $K_{nm}$. The approximate kernel matrix can serve as a replacement of the original kernel matrix for improved computational efficiency for different kernel methods.

Kernel methods in machine learning look for functions $f$ in the reproducing kernel Hilbert space to be adapted to data
\[
    f(x)
    = \sum_{i=1}^n \alpha_i \langle \phi(x_i), \phi(x) \rangle_{\Hcal}
    = \sum_{i=1}^n \alpha_i k(x_i, x)
\]
where $\{\alpha_i\} \in \R^n$ are parameters. The Nyström method may also be defined by restricting these functions to lie in the linear span of the $m$ subsampled data points $\{\phi(x_r)\}_{r \in S}$, while using the full dataset of $n$ points for estimation of the unknown parameters \citep{rudi2015less}. For fixed $S$ the linear span of $\{\phi(x_r)\}_{r \in S}$ is a closed subspace of $\Hcal$ and so is a Hilbert space, which we will denote by $\Hcal_S$ \citep{bollobas1999linear}. In other words, one looks for functions of the form
\[
    f(x) = \sum_{r \in S} \alpha_r \langle \phi(x_r), \phi(x) \rangle_{\Hcal} =
    \sum_{r \in S} \alpha_r k(x_r, x)
\]
that solve an estimation problem based on all data points $\{x_i\}_{i=1}^n$, such as an empirical risk minimization procedure.

After drawing the $n$ observations $\{x_i\}_{i=1}^n$ independently from $\p_X$, the subset of $m$ data points $\{x_r\}_{r \in S} = \{x_{r_1}, x_{r_2}, ..., x_{r_m}\}$ is randomly selected according to a specified distribution that may depend on the observed values $p(S | \{x_i\}_{i=1}^n)$. Before the data points are observed the elements in the subset are random variables $\{X_{r_1}, X_{r_2}, ..., X_{r_m}\}$. For notational convenience we will assume that the data points are reordered after the subsampling so that $\{x_r\}_{r \in S} = \{x_1, x_2, ..., x_m\}$.

Kernel PCA may be obtained by appealing to the $\ell^2(\R)$ representation of a separable real Hilbert space and arranging the data points in $\Hcal$ in a data matrix $\krn$ with one data point occupying a row, which may then have an infinite number of columns. With zero-mean data in feature space the principal components are then the eigenvectors of $\frac{1}{n}\krn^T\krn$ and the kernel matrix can be written as $K = \krn\krn^T$. The mean can be subtracted in the RKHS (the feature space) through \citep{scholkopf1998nonlinear}
\begin{equation*}
    K' = (\krn - \one_{n} \krn )(\krn - \one_n \krn)^T
       = K - \one_{n}K - K\one_{n} + \one_{n}K\one_{n}
\end{equation*}

where $\one_{n}$ is a matrix for which $(\one_{n})_{i,j} = \frac{1}{n}$. The eigenvalues of $K' = Q \Lambda Q^T$ scaled by $\frac{1}{n}$ then measure the variance of the data projected onto each individual principal component. Its eigenvectors $Q$ are proportional to the principal scores \--- the principal scores are given by $S = Q \Lambda^{1/2}$. By the singular value decomposition $\krn - \one_n \krn = Q \Sigma E^T$, where $\Lambda = \Sigma^2$, the principal scores of a \emph{new} data point $x^*$ which is centred in feature space is given by
\begin{align*}
    w^* &= ((\phi(x^*) - \mathbf{1}_n\krn)E)^T
         = ((\phi(x^*) - \mathbf{1}_n\krn) (\krn - \one_n \krn)^T Q \Lambda^{-1/2})^T \\[1em]
        &= ((\kappa(x^*)^T - \kappa(x^*)^T \one_n - \mathbf{1}_n K + \mathbf{1}_n K \one_n) Q \Lambda^{-1/2})^T \\[1em]
        &=  \Lambda^{-1/2} Q^T (\kappa(x^*) - \one_n \kappa(x^*) - K \mathbf{1}_n  +  \one_n K \mathbf{1}_n)
        =:  \Lambda^{-1/2} Q^T \kappa'(x^*)
        =  S^{-1} \kappa'(x^*)
\end{align*}
where $\phi(x_i)$ is an element in $\ell^2(\R)$ as a row vector, $\mathbf{1}_n$ is a length-$n$ column vector with each element equal to $\frac{1}{n}$ and $\kappa(x) = (k(x_1, x), k(x_2, x), ..., k(x_n, x))^T$.

Using this formula to calculate the scores for the \emph{original} data points we get that $\kappa'(x^*)$ becomes $K'$ and obtain $w^{*T} = K' Q \Lambda^{-1/2} = Q \Lambda Q^T Q \Lambda^{-1/2} = Q \Lambda^{1/2}$ and so as expected we recover the previous expression for the principal scores.

When applying PCA to a real-world problem it is often appropriate to normalize the input variables to have variance 1, so as to make the analysis independent of arbitrary changes of units in the data. Otherwise the variables with higher variance will also dominate the principal components and comparisons between variables become difficult. This normalization will often also be appropriate for kernel PCA and we do this for the experimental analysis (Section \ref{sec:exp}). The centring of variables in the feature space does not guarantee that the input variables become centred.

Multi-dimensional scaling (MDS) finds a lower-dimensional representation of data from a matrix of distances between data points \citep{hout2013multidimensional}. MDS is equivalent to kernel PCA when the kernel is \emph{isotropic}, i.e. on the form $f(\|x - y \|)$ for some function $f$ \citep{williams2002connection}. Therefore, theoretical or practical results for kernel PCA are often also applicable to MDS.

The approximate eigenvalues and eigenvectors from the original Nyström method in Equations (\ref{eq:nysval}) and (\ref{eq:nysvec}) may be used to define an approximate kernel PCA. However, these approximate eigenvectors are not orthogonal and do not yield uncorrelated principal scores, so they do not define true PCA, and the eigenvalues do not describe the variance captured by the principal components, since they are simply the eigenvalues of $K_{mm}$ scaled by a factor $\frac{n}{m}$. There is a need for another way to derive kernel PCA with the Nyström method.

\vspace{32pt}

\section{Kernel PCA with the Nyström method} \label{sec:kpca}

\vspace{12pt}

In this section we present kernel PCA with the Nyström method, which provides an efficient and flexible technique for non-linear PCA. We present the corresponding quantities that are defined for linear PCA and are useful for data exploration and application of the method in downstream tasks

\begin{enumerate}[(1)]
    \item a set of orthogonal principal components with unit length in the linear span of the subsampled data points in $\Hcal$ (denoted $\Hcal_S$),
    \item the variance of the data along each of these directions, termed the explained variance,
    \item the reconstruction error of the data onto the principal components,
    \item a set of uncorrelated principal scores with the weightings of the data points on the principal components, and,
    \item the principal scores of a new data point with respect to the existing principal components
\end{enumerate}

For standard kernel PCA (2) and (3) are the same, but with the Nyström method they are different, since the principal components will not span the entire data.

We first present the principal components, explained variance and scores for a dataset in the following theorem

%\vspace{12pt}

\begin{theorem}[Nyström kernel PCA] \label{thm:nystrompca}

Let $(\widetilde{\lambda}_j, v_j)$ be the eigenpairs and $V\widetilde{\Lambda}V^T$ be the eigendecomposition of
\begin{equation*}
    \widetilde{K}' = \frac{1}{n}K_{mm}^{\prime \,-1/2} K'_{mn} K'_{nm} K_{mm}^{\prime \, -1/2}
\end{equation*}
where
\begin{align*}
    K'_{mn} &= 
    K_{mn} - K_{mn}\one_n 
    - \one_n^{m, n} \widetilde{K}
    + \one_n^{m, n} \widetilde{K} \one_n \\[0.5em]
    K'_{mm} &=
    K_{mm} - \one_n^{m, n} K_{nm}
    - K_{mn} \one_n^{n, m}
    + \one_n^{m, n} \widetilde{K} \one_n^{m, n}
\end{align*}
with $\widetilde{K} = K_{nm}K_{mm}^{-1}K_{mn}$ and where $\one_n$, $\one_n^{n, m}$ and $\one_n^{m, n}$ are $n \times n$, $n \times m$ and $m \times n$ matrices respectively with each element equal to $\frac{1}{n}$.

The perpendicular intersecting lines $\phi_0 + \langle\widetilde{\phi}_j \rangle, \, j = 1, 2, ..., m$ in $\Hcal_S$ along which the variance of the data is successively maximized, where the orthonormal vectors $\{ \widetilde{\phi}_j\}_{j=1}^m$ are termed the principal components, are given by
\begin{align*}
    \phi_0 &= \frac{1}{n}K_{nm}K_{mm}^{-1}\kappa_m(x) \\[0.5em]
    \widetilde{\phi}_j &= \sum_{k=1}^m u_{j,k}
        \left( k(x_k, x) - \phi_0 \right)
\end{align*}
and the variances along these directions are $\{ \widetilde{\lambda}_j \}_{j=1}^m$, termed the principal values or explained variance, where $\kappa_m(x) = (k(x_1, x), k(x_2, x), ..., k(x_m, x))^T$, $u_j = K_{mm}^{\prime \, -1/2} v_j$ and $U = K_{mm}^{\prime \, -1/2} V$.

The projection coefficients of the centred data points onto the principal components, termed the principal scores, are given by
\[
    W = K'_{nm}U = K'_{nm} K_{mm}^{\prime \, -1/2} V
\]
where each row of $W$ contains the scores of one data point onto the principal components. The principal scores of a new data point $x^*$ is given by
\[
    w^* = U^T (\kappa_m(x^*) - K_{mn} \mathbf{1}_n
                 - \one_n^{m, n}K_{nm}K^{-1}_{mm} \kappa_m(x^*)
                 + \one_n^{m, n} \widetilde{K} \mathbf{1}_n)
     = U^T \, \widetilde{\kappa}(x^*)
\]
where $\mathbf{1}_n$ is a length-$n$ column vector given by $\mathbf{1}_n = (\frac{1}{n}, \frac{1}{n}, \; ..., \; \frac{1}{n})^T$.

\end{theorem}

The principal components can be seen as defining new variables through linear combinations of the existing variables that have successively maximized variance and that are uncorrelated. The values of these new variables are given by the principal scores, which represent the data in a new coordinate system defined by the principal components as a new basis for the space. As such, the principal scores can be used as a drop-in replacement for the original data in arbitrary supervised or unsupervised learning methods, including after removing the scores corresponding to principal components with smaller eigenvalues. Please see Section \ref{sec:appl} for an example of this.

To see that these new variables are uncorrelated also with the Nyström method we note that
\[
    W^TW
    = V^T K_{mm}^{\prime \, -1/2}K'_{mn} K'_{nm} K_{mm}^{\prime \, -1/2} V
    = n V^T V \widetilde{\Lambda} V^T V = n \widetilde{\Lambda}
\]
which is a diagonal matrix.

The scores of new data points are important when measuring the accuracy of PCA with a test set of hold-out data points, for example using the reconstruction error (Section \ref{sec:exp}), or when applying PCA as a preprocessing step for supervised learning methods and one wishes to create predictions for new data points, such as in principal component regression (Section \ref{sec:appl}).

The computational complexity of the method is $\Ocal(nm^2)$ in time, which is the same as the Nyström method applied to regression. Centring of the matrix $K_{mn}$ can be accomplished in $\Ocal(m^3 + nm)$ operations, and so the centring in the proposed method adds no additional time requirements to the dominant $\Ocal(nm^2)$ factor. We refer to the software implementation for full details\footnote{{\fontsize{7.9pt}{7.9pt}\selectfont \url{https://github.com/fredhallgren/nystrompca/blob/main/nystrompca/algorithms/nystrom_KPCA.py}}}.

The Nyström method approximates the corresponding full method, so when $m = n$ we should recover standard kernel PCA. In this case $\widetilde{K} = K K^{-1} K = K$ and as expected $K'_{mm} = K'_{nm} = K'$ and 
\[
    K_{mm}^{\prime \, -1/2} K'_{mn} K'_{nm} K_{mm}^{\prime \, -1/2} = K'
\]

and the scores are equal to $W = K^{\prime \, 1/2}V = \sqrt{n} \, V \widetilde{\Lambda}^{1/2} V^T V = \sqrt{n} \, V \widetilde{\Lambda}^{1/2} = Q\Lambda^{1/2}$, which we know to be the scores for standard kernel PCA.

The smallest $m-d$ Nyström eigenvalues $\sum_{j=d+1}^m \widetilde{\lambda}_j$ measure the residual variance of the data points \emph{within} $\Hcal_S$ and correspond to the reconstruction error $\frac{1}{n} \sum_{i=1}^n \|P_{\widetilde{V}_d} z'_i - P_{\Hcal_S}z'_i \|_{\Hcal}^2$, where $\widetilde{V}_d = \mathrm{span}\{\{\widetilde{\phi}_k\}_{k=1}^d\}$. The full reconstruction error with respect to the top $d$ Nyström principal components is given by
\begin{equation} \label{eq:fullrec}
    R_n(\widetilde{V}_d) = \frac{1}{n}\sum_{i=1}^n \|z'_i - P_{\widetilde{V}_d}z'_i \|^2_{\Hcal} = \frac{1}{n}\tr(K') - \sum_{j=1}^d \widetilde{\lambda}_j
\end{equation}
where $\tr(\cdot)$ is the trace and $\frac{1}{n}\tr(K')$ is the variance of the full dataset in $\Hcal$. From Theorem~\ref{thm:nystrompca} above we know that this is the smallest reconstruction error among all $d$-dimensional subspaces in $\Hcal_S$.

Calculation of this quantity is $\Ocal(n^2)$ due to the centring of $K$. However, it can be approximated for example by subtracting the mean of $K_{nm}$ instead of the mean of $K$, which becomes $\Ocal(nm)$. This is included as an option in the software package accompanying the paper. Please see Section \ref{sec:exp} on page~\pageref{sec:exp} for further details.
%\footnote{\url{https://github.com/fredhallgren/nystrompca}}.

Note that the reconstruction error above in Equation~(\ref{eq:fullrec}) is slightly different from the reconstruction error of the uncentred data points with respect to the affine subspace $\phi_0 + \widetilde{V}_d$, which becomes $\frac{1}{n}\sum_{i=1}^n \|(z_i - \phi_0) - P_{\widetilde{V}_d}(z_i - \phi_0)\|_{\Hcal}^2 = \frac{1}{n}\sum_{i=1}^n \|(z_i - \phi_0) - P_{\widetilde{V}_d}z'_i\|_{\Hcal}^2$. Both reconstruction errors are at a minimum for the proposed method.

Another quantity of interest for purposes of comparison is the reconstruction error of the full dataset on the eigenspace of the subset of $m$ data points. Creating PCA from a random subset of $m$ data points to describe the full dataset will be termed \emph{Subset PCA}. We use the same subspace translation as for the Nyström method \--- that is to say we centre the data using the mean of the $n$ data points projected onto $\Hcal_S$. This ensures that the amount of variance captured is the same whether we project the centred data onto the principal components, or the uncentred data onto the lines translated from the origin. The principal components will then be given by, for $j = 1, 2, ..., m$
\[
    \hat{\phi}^{m,n}_j =  \sum_{k=1}^m u^m_{j,k} (k(x_k, x) - \phi_0)
\]
where $u^m_j$ is the $j$th eigenvector of $\frac{1}{m}K'_{mm}$. The variance of the full data captured by these principal components and the associated reconstruction error are presented in the following theorem

\vspace{12pt}

\begin{theorem}[Subset PCA] \label{thm:subsetpca}
The variance of the dataset $\{ \phi(x_i) \}_{i=1}^n$ along the $j$th principal component $\hat{\phi}^{m,n}_j$ is given by
\[
    \hat{\lambda}_j^{m,n}
    = \frac{1}{n}\sum_{i=1}^n \|P_{\hat{\phi}_j^{m,n}}z'_i\|_{\Hcal}^2
    = \frac{1}{n \cdot m \hat{\lambda}_j^m} u_j^{m \, T} K'_{mn}K'_{nm}u_j^m
\]
where $(\hat{\lambda}^m_j, u^m_j)$ is the $j$th eigenpair of $\frac{1}{m}K'_{mm}$.

The reconstruction error of the full dataset onto the corresponding $d$-dimensional PCA subspace is
\[
    R_n(\hat{V}_d^m)
    = \frac{1}{n}\sum_{i=1}^n \|z'_i - P_{\hat{V}_d^m}z'_i\|_{\Hcal}^2
    = \frac{1}{n} \tr(K')
    - \frac{1}{n \cdot m} \tr(K'_{nm}U^m_d \Lambda_d^{m \,-1} U^{m \, T}_d K'_{mn})
\]
where $U^m_d \Lambda_d^m U^{m \, T}_d$ is the truncated eigendecomposition of $\frac{1}{m}K'_{mm}$.

\end{theorem}

\vspace{12pt}

As expected, if $n = m = d$ then the reconstruction error is zero.

The method proposed in this section for efficient kernel PCA can also be applied to improve the scalability of MDS when these two methods are equivalent, as outlined in Section~\ref{sec:background}.

\vspace{16pt}
\newpage
\section{Prelude: A special case} \label{sec:prelude}

\vspace{12pt}

Before studying the statistical accuracy of kernel PCA with Nyström method through a confidence bound we present a majorization relation between Nyström and Subset PCA and consider the special case when the PCA dimension equals the number of subsampled data points, $d = m$. In this case the reconstruction error for the Nyström method is the same as Subset PCA, both for the empirical and true reconstruction errors.

We first present the majorization relation in the following proposition. It tells us with one concise formula that for any PCA dimension strictly less than the number of data points in the subset, kernel PCA with the Nyström method will always capture at least as much variance as PCA created directly from the subset, but that there will be no improvement when $d=m$.

\begin{proposition} \label{prop:maj}
We have the following majorization relation for the empirical error
\[
    (\widetilde{\lambda}_1, \widetilde{\lambda}_2, ..., \widetilde{\lambda}_m)
    \; \succ \;
    (\hat{\lambda}^{m,n}_1, \hat{\lambda}^{m,n}_2, ..., \hat{\lambda}^{m,n}_m)
\]

\end{proposition}

The majorization is strict in the sense that $\widetilde{\lambda}_{<d} > \hat{\lambda}_{<d}^{m,n}$ for $d < m$, by the assumption of a continuous data distribution.

A direct consequence of the proposition is that
\[
    R_n(\widetilde{V}_m) = R_n(\hat{V}_m^m)
\]

For the true reconstruction error we consider the case where the sampling of the Nyström subset occurs independently of the values of the data points

\begin{proposition} \label{prop:trueeq}
Let $d=m$ and let the Nyström subset be sampled according to $p(S \, | \, x_1, x_2, ..., x_n)$. Then if
\[
    p(S \, | \, x_1, x_2, ..., x_n) = p(S)
\]

i.e. the subsampling is independent of the data, we have
\[
    R(\widetilde{V}_m) = R(\hat{V}_m^m)
\]
\end{proposition}

The above proposition includes the common case of uniform sampling for the Nyström subset. It holds whether the $n$ data points are considered fixed or unobserved.

From the above propositions we can conclude that if retaining all the Nyström principal components then there is no gain in accuracy compared to Subset PCA from the perspective of the reconstruction error. However, for a smaller PCA dimension the Nyström method will perform strictly better than PCA directly from the subset. A more precise treatment of its accuracy for arbitrary dimensions is the subject of the next section. %Furthermore, other strategies for sampling of the subset may lead to a higher accuracy for the Nyström method even when $d=m$.

\vspace{32pt}

\section{Statistical accuracy of Nyström kernel PCA} \label{sec:accuracy}

\vspace{12pt}

In this section we provide a finite-sample confidence bound on the empirical reconstruction error of kernel PCA with the Nyström method versus the one for full kernel PCA. In line with strictly all results on the statistical accuracy on standard kernel PCA we assume that data has zero mean in feature space.

The confidence bound allows for measuring the accuracy of Nyström kernel PCA for a specific dataset, specified in the familiar language of confidence intervals, applied to the amount of variance that is left over after representing the dataset in terms of a subset of principal components.

The actual difference between the reconstruction errors of the Nyström method and standard kernel PCA for a dataset is given by
\[
    R_n(\widetilde{V}_d) - R_n(\hat{V}_d)
    = \frac{1}{n}\tr(K) - \sum_{j=1}^d \widetilde{\lambda}_j - \sum_{j=d+1}^m \hat{\lambda}_j^n
    \; = \; \hat{\lambda}_{<d}^n \; - \; \widetilde{\lambda}_{<d}
\]
However, the eigenvalues $\hat{\lambda}_j^n$ of $\frac{1}{n}K$ are not available \--- if they were there would be no need to apply the Nyström method. When the Nyström method is being considered for a problem then the size of the data $n$ is very large and calculating the full kernel matrix $K$, let alone its eigendecomposition, is prohibitively expensive.

At a minimum, any measure of accuracy should not be more computationally demanding than the method itself, which is $\Ocal(nm^2)$. We present a bound that does not require that we have observed the entire dataset, only the subset $x_1, x_2, ..., x_m$. It takes $\Ocal(m^3)$ time to calculate and is $\Ocal(m^2)$ in memory. It holds for any subsampling distribution.

\vspace{8pt}

% NOTE: calling something a probability already means it's unobserved
\begin{theorem}[Confidence bound] \label{thm:prob}
With confidence at least $1 - 2e^{-\delta}$ \emph{(}or, with probability at least $1 - 2e^{-\delta}$ with respect to $\{x_i\}_{i=m+1}^n$ across repeated samples of $\{x_r\}_{r=1}^m$\emph{)}, where $B := \sup_x k(x,x)$, $\{\hat{\lambda}_j^m\}_{j=1}^m$ are the eigenvalues of the kernel matrix $\frac{1}{m}K_{mm}$ from the Nyström subset, $\hat{\lambda}_0^m$ and $\hat{\lambda}_{m+1}^m$ are defined to be $+\infty$ and $-\infty$ respectively, and
\begin{align*}
    D \, \, &:=
     \frac{n-m}{n}
\frac{2B\sqrt{\delta}}{\sqrt{n-m}}
    \\[1em]
    D_j &:=
     \frac{(2D)^2}{\min\{
    \hat{\lambda}_{j-1}^m - \hat{\lambda}_j^m,
    \hat{\lambda}_j^m - \hat{\lambda}_{j+1}^m
  \}^2}
    \wedge 1
\end{align*}
we have
\[
    R_n(\widetilde{V}_d) - R_n(\hat{V}_d) \le
\sum_{j=1}^d
\hat{\lambda}^m_j \cdot D_j + D \cdot \max_{1 \le k \le d} D_k
\]

\end{theorem}

\vspace{12pt}

% Explanation of confidence and probability
The bound is stated in terms of a probability with respect to future unknown realizations of the data $\{x_i\}_{i=m+1}^n$, which holds across infinite repetitions of the experiment yielding the observed data $\{x_r\}_{r=1}^m$ as in the frequentist construction of hypothesis tests or confidence intervals \citep{neyman1933ix, neyman1937outline}. Equivalently, the bound may be interpreted solely as a confidence, both with respect to the observed data and with respect to hypothesized future realizations of the unobserved data. Stated differently, if we observe the subset $\{x_r\}_{r=1}^m$, calculate the bound with some specified confidence level $\alpha$, say 95 \%, and observe the future data $\{x_i\}_{i=m+1}^n$, then the realized difference in reconstruction errors will lie within the bound at least 95 \% of the time if we repeat this procedure indefinitely.

The behaviour of the bound is as one would expect from a measure of the statistical accuracy of the Nyström method compared to the full method \--- it decreases as $m$ increases, \emph{ceteris paribus}, and becomes zero if $n=m$. It increases with the dimension $d$ of the PCA subspaces that are being compared. %A slower decay of the eigenvalues of $K_{mm}$ leads to a larger bound.

Application of the bound does not require that we have observed the entire sample. For example, if data is generated sequentially and iid from $p_X(x)$ then picking the first $m$ points for the Nyström subset is equivalent to sampling all points and then selecting $m$ points uniformly (in the sense that the data points in the subset have the same distribution in both instances).

If data is stored on disk, and reading from disk is expensive, then only $m$ records need to be read in order to calculate the bound, assuming this can be done in such a way as to respect the sampling distribution of the subset of data points\footnote{In many implementations of the SQL query language, including MySQL and PostgreSQL, this would correspond to appending \texttt{LIMIT(\textit{m})} to the end of the query, which interrupts it after finding the first $m$ records \citep{beaulieu2020learning}}.

The bound becomes infinite if $k(x,x)$ is not bounded for all $x$. One may create a bounded kernel from an unbounded one through the transformation
\begin{equation} \label{eq:knorm}
    k'(x,y) := \frac{k(x,y)}{\sqrt{k(x,x)k(y,y)}}
\end{equation}
which has $\sup_x k'(x,x) = 1$, although the transformed kernel has somewhat different characteristics and induces a different RKHS. The transformation corresponds to scaling all feature vectors to have norm 1.

\vspace{12pt}

\subsection*{Proof outline}

A proof outline is as follows. Please see the appendix for a full proof.

\begin{enumerate}[1.]
    \item Rewrite the difference in reconstruction errors in terms of the eigenpairs of the empirical operators $C_n$ and $C_m$, to obtain
    \vspace{-12pt}
        \[
            R_n(\widetilde{V}_d) - R_n(\hat{V}_d)
            \le
            \sum_{j=1}^d
            \hat{\lambda}_j^n
            \left(
            1 -
            \langle
            \hat{\phi}^n_j,
            \hat{\phi}_j^m
            \rangle_{\Hcal}^2
            \right)
        \]
    \vspace{-20pt}
    \item Apply the Davis-Kahan theorem to convert the angle between the eigenvectors into a difference between successive eigenvalues of $C_m$ and the norm of the difference between the empirical operators $\|C_n - C_m\|_{\hs}$
    \item Convert the unknown eigenvalues $\hat{\lambda}^n_j$ into the ones based on the observed data $\hat{\lambda}^m_j$ plus the difference $\|C_n - C_m\|_{\hs}$, using Lidskii's inequality
    \item Now $\|C_n - C_m\|_{\hs}$ is the only unknown quantity left. Split up the empirical operators into two independent ones through
        \[
		\|C_n - C_m\|_{\hs} = \frac{n-m}{n}\|C_{n-m} - C_m\|_{\hs}
        \]
    where $C_{n-m} = \frac{1}{n-m}\sum_{i=m+1}^n z_i \otimes z_i$
    \item Apply Hoeffding's inequality in Banach spaces, obtaining
        \[
		\p\left( \|C_{n-m} - C_m\|_{\hs}
                     \le 2B\sqrt{\delta} / \sqrt{n-m} \right) \ge 1 - 2e^{-\delta}
        \]
        \vspace{-24pt}
\end{enumerate}

\vspace{32pt}

\subsection{A corollary} \label{sec:cor}

From Lemma~\ref{lem:opbound} on page~\pageref{lem:opbound} which is used in the proof of the confidence bound one can deduce sharper versions of Theorem 7 and Propositions 10 and 11 from \cite{rosasco2010learning}, by a factor $1/\sqrt{2}$ or $1/2$. These follow since the covariance operator $C$ and its empirical equivalent $C_n$ are positive, and then by the lemma their difference $\|C - C_n\|_{\hs}$ is bounded by $\sqrt{2}\sup_x k(x,x)$, rather than $2 \sup_x k(x,x)$.

For Theorem 7, the sharper result states that with probability at least $1 - 2e^{-\delta}$ we have
\[
    \|C - C_n\|_{\hs} \le \frac{2B\sqrt{\delta}}{\sqrt{n}}
\]
The sharper version of Proposition 10 states that with probability $1 - 2e^{-\delta}$
\begin{align*}
    \sum_{j=1}^{\infty} \left( \lambda_j - \hat{\lambda}_j^n \right)^2
    \le \frac{4B^2\delta}{n}
\,\,\,\,\,\,\,\,
\,\,\,\,\,\,\,\,
\,\,\,\,\,\,\,\,
\,\,\,\,\,\,\,\,
\,\,\,\,\,\,\,\,
    \sup_j | \lambda_j - \hat{\lambda}_j^n |
    \le \frac{2B\sqrt{\delta}}{\sqrt{n}}
\end{align*}

And for Proposition 11 we obtain that also with probability $1 - 2e^{-\delta}$

\[
    \left| \sum_{j=1}^{\infty} \lambda_j - \sum_{j=1}^n \hat{\lambda}_j^n \right|
    = \left| \tr(C) - \tr(C_n) \right|
    \le \frac{2B\sqrt{\delta}}{\sqrt{n}}
\]

A number of other results can also be sharpened using the same technique, including, but not limited to, Theorem 2 in \cite{de2005learning}, Theorem 1 in \cite{de2005model}, Lemma 1 in \cite{zwald2005convergence}, Theorem 6.2 in \cite{giraldo2014measures}, Theorem 4.2 in \cite{bouvrie2012kernel} and Lemma 4.1 in \cite{blanchard2019concentration}.

%\vspace{16pt}
\newpage
\section{Experimental analysis} \label{sec:exp}

\vspace{12pt}

In this section we illustrate the method and bound through experiments on real-world datasets with different kernel functions. We first compare the proposed method to a number of other unsupervised learning methods by measuring the reconstruction error on hold-out datasets. We then evaluate the bound and compare it to the actual errors and the errors for Subset PCA.

The methods and experiments are implemented in the Python programming language and the source code is available at {\small\url{https://github.com/fredhallgren/nystrompca}}. The package can be installed with one simple command using the Python package manager\footnote{{\Small\texttt{pip install nystrompca}}}. It includes a command-line tool to run the different experiments with different parameter values and kernel functions.

For purposes of reproducibility the computer experiments allow for setting the random seed of the pseudo-random number generator \citep{robert2013monte}, to produce exactly the same results every time the experiments are run. Other than the random sampling of the Nyström subset, randomness is also present in the splitting of data into training and test sets.

The principal components are unique only up to a sign, so in the package we switch the sign of the scores and components such that the range of values in each dimension of the scores is mostly positive. This will ensure that we will get exactly the same values for the scores and components every time we run the algorithm.

We use different datasets from the UCI Machine Learning Repository \citep{dua2019uci}. Dimensionality reduction can be particularly important for high-dimensional data, so we include a number of such datasets. We use the simulated \texttt{magic} gamma telescope dataset, the \texttt{yeast} dataset, containing cellular protein location sites for fungi, the \texttt{cardicotocography} dataset, with heart measurements, the \texttt{segmentation} dataset containing various data on images, the \texttt{drug} dataset with personality traits and drug consumption, the \texttt{digits} dataset with flattened $8 \times 8$ pixel grayscale images, and two bag-of-words datasets with bag-of-words vectors of articles from \texttt{www.dailykos.com} and NIPS papers, respectively. We tabulate some information on the datasets used in one or both of the experiments below in Table \ref{tab:datasets}, where the number of attributes is before any data transformation. For comparability we cut each dataset to 1000 data points when running the experiments. For both experiments we sample the Nyström subset uniformly without replacement and we use the same sampled subset for both Nyström PCA and Subset PCA.

\renewcommand{\arraystretch}{1.3}
\begin{table}[h]
\vspace{16pt}
    \caption{Datasets used}
    \label{tab:datasets}
\begin{center}
    \small
\begin{tabular}{cccc}
    \toprule
  \;\;\;\;\;\;\;\;\;  & \textit{Dataset} & \textit{Data size} & \textit{Number of attributes} \\
    \midrule
  1 & \texttt{magic}            &   19020 &    11 \\
  2 & \texttt{yeast}            &    1484 &     8 \\
  3 & \texttt{cardiotocography} &    2126 &    23 \\
  4 & \texttt{segmentation}     &    2310 &    19 \\
  5 & \texttt{drug}             &    1885 &    32 \\
  6 & \texttt{digits}           &    5620 &    64 \\
  7 & \texttt{dailykos}         &    3430 &  6906 \\
  8 & \texttt{nips}             &    1500 & 12419 \\
    \bottomrule
\end{tabular}
\end{center}
\vspace{6pt}
\end{table}

We convert ordinal variables to integers and categorical variables to discrete ones through one-hot encoding. We treat discrete numerical variables in the data as continuous for the purposes of PCA. We remove any date or time variables. We also remove variables that are constant. These will differ depending on how many data points we include in the total dataset when we run the experiments.

We normalize the input data to have mean zero and variance one. Note that this does not mean that data has zero mean in the feature space. As previously mentioned, normalizing the input data makes the analysis independent of the units used to measure the variables and unaffected by the scale of the variables, which may otherwise dominate the PCA results. Furthermore, it makes it easier to compare results across different datasets and kernel functions and can make the same kernel parameters appropriate for different datasets.

We cut eigenvalues that are smaller than $10^{-12}$ when performing matrix inversions to improve the condition number of the matrix. We also remove any negative eigenvalues \--- in theory all kernel matrices will be positive definitive, however numerical inaccuracies may occasionally lead to small negative eigenvalues in practice.

We use three different kernel functions, the radial basis functions (RBF), polynomial and Cauchy kernels, summarized below in Table \ref{tab:kernels}. The software package includes a number of additional kernel functions that can be used when running either of the experiments.

\renewcommand{\arraystretch}{2}
\begin{table}[h]
\vspace{16pt}
    \caption{Kernel functions used}
    \label{tab:kernels}
\begin{center}
\small
\begin{tabular}{cccc}
    \toprule
    \textit{Kernel}
    & \textit{Functional form} $k(x,y)$
    & \textit{Parameters}
    & \textit{Bound} $\sup_x k(x,x)$ \\
    \midrule
    RBF
    & $\exp\left\{-\frac{\|x - y\|^2}{\sigma^2}\right\}$
    & $\sigma \in \R_+ \setminus \{0\}$
    & 1\\
    Polynomial
    & $\left(\langle x, y \rangle + R\right)^d$
    & $R \in \R, \; d \in \mathbb{N}$
    & $\infty$\\
    Cauchy
    & $\frac{1}{1 + \|x-y\|^2/\sigma^2}$
    & $\sigma \in \R_+ \setminus \{0\}$
    & 1\\
    \bottomrule
\end{tabular}
\end{center}
\vspace{6pt}
\end{table}

\emph{}

%\vspace{6pt}

\subsection{Methods comparison} We compare the proposed method to other unsupervised learning techniques to evaluate its behaviour. We compare with linear PCA, full kernel PCA, subset PCA, sparse PCA, which finds sparse eigenvectors \citep{wang2016statistical}, locally linear embeddings (LLE), that fits a lower-dimensional manifold to the data \citep{roweis2000nonlinear} and independent component analysis (ICA), which extracts signals that are independent \citep{hyvarinen2000independent}. We run the methods for all the datasets in Table \ref{tab:datasets} above. We split each dataset randomly in half, fitting the methods on one half and then evaluating them on the other half. We compare the fraction of variance captured for the different methods for different dimensions. Note that kernel PCA and Nyström kernel PCA measure the variances captured in the RKHS and not in the input space.

For this experiment we only display the results for the RBF kernel and we use a Nyström subset of size $m = 100$.
For the first six datasets we calculate the bandwidth parameter as the median distance between pairs of data points, which is a common heuristic for the RBF kernel \citep{garreau2017large}. Using all pairs of data points is quadratic in the total number of data points, so we only use the data points in the Nyström subset. For the last two datasets we set the bandwidth parameter to $\sigma = 500$, which was manually tuned. Please see Table~\ref{tab:compres} for the full results, where we have set the random seed to $1$. Sparse PCA is computationally demanding for very high-dimensional data, so we don't run it for all the datasets.

\renewcommand{\arraystretch}{1.7}
\begin{table}[h] \vspace{16pt}
    \label{tab:compres0}
    \caption*{\emph{Table continues on the next page}}
\begin{center}
\Small
\renewcommand{\arraystretch}{1.09}
\begin{tabular}{P{1cm}P{0.3cm}P{1.47cm}P{1.73cm}P{1.5cm}P{1.52cm}P{1.52cm}P{1.52cm}P{1.52cm}}
    \toprule

\textit{Dataset}    & $d$ & \textit{Subset PCA} & \textit{Nyström PCA}  & \textit{Kernel PCA} & \textit{Linear PCA} & \textit{Sparse PCA} & \textit{LLE} & \textit{ICA} \\
    \midrule

    \texttt{magic} \\

 & 1 &  0.2401        &  0.2494   &  0.2500   &  0.5089      & 0.5027 & 0.0660 &  0.0937   \\ 
 & 2 &  0.3612        &  0.3746   &  0.3760   &  0.6223      & 0.6353 & 0.1719 &  0.1798   \\
 & 3 &  0.4187        &  0.4417   &  0.4450   &  0.7170      & 0.7303 & 0.2828 &  0.2630   \\
 & 4 &  0.4760        &  0.5063   &  0.5102   &  0.7926      & 0.7341 & 0.4151 &  0.3488   \\
 & 5 &  0.5387        &  0.5653   &  0.5690   &  0.8575      & 0.7798 & 0.5243 &  0.3488   \\
 & 6 &  0.5731        &  0.6052   &  0.6093   &  0.9249      & 0.8395 & 0.6210 &  0.5535   \\
 & 7 &  0.6047        &  0.6372   &  0.6423   &  0.9648      & 0.8658 & 0.7499 &  0.6962   \\
 & 8 &  0.6243        &  0.6622   &  0.6688   &  0.9808      & 0.8981 & 0.7954 &  0.6962   \\
 & 9 &  0.6450        &  0.6843   &  0.6911   &  0.9979      & 0.9588 & 0.8997 &  0.6962   \\
 & 10&  0.6613        &  0.7042   &  0.7115   &  1.0000      & 0.9803 & 1.0000 &  1.0000   \\

    \vspace{1pt}
    \texttt{yeast} \\

   &1  &   0.1184      &  0.1400     &  0.1401    &  0.1182    & 0.1150  & 0.0527 & 0.0431  \\
   &2  &   0.2350      &  0.2562     &  0.2566    &  0.2207    & 0.2132  & 0.1367 & 0.0911  \\
   &3  &   0.3419      &  0.3714     &  0.3727    &  0.3144    & 0.3050  & 0.1978 & 0.1223  \\
   &4  &   0.4138      &  0.4418     &  0.4437    &  0.4179    & 0.4111  & 0.2752 & 0.1977  \\
   &5  &   0.4532      &  0.4901     &  0.4922    &  0.5058    & 0.4963  & 0.2892 & 0.2598  \\
   &6  &   0.5066      &  0.5362     &  0.5385    &  0.5876    & 0.5992  & 0.3815 & 0.3149  \\
   &7  &   0.5438      &  0.5681     &  0.5716    &  0.6462    & 0.6472  & 0.4350 & 0.3655  \\
   &8  &   0.5724      &  0.5985     &  0.6020    &  0.6877    & 0.7012  & 0.5215 & 0.4145  \\
   &9  &   0.5997      &  0.6329     &  0.6364    &  0.7520    & 0.7665  & 0.5371 & 0.4569  \\
   &10 &   0.6279      &  0.6475     &  0.6512    &  0.7949    & 0.8026  & 0.6000 & 0.5056  \\

    \vspace{1pt}
    \texttt{cardiotocography} \\

  & 1  &    0.1398      &  0.1422    &  0.1430     &  0.2173    &  -    & 0.0251  & 0.0260  \\
  & 2  &    0.2100      &  0.2351    &  0.2370     &  0.3687    &  -    & 0.0645  & 0.0521  \\
  & 3  &    0.2946      &  0.3105    &  0.3133     &  0.4636    &  -    & 0.0862  & 0.0765  \\
  & 4  &    0.3496      &  0.3674    &  0.3714     &  0.5364    &  -    & 0.1163  & 0.1019  \\
  & 5  &    0.3897      &  0.4091    &  0.4139     &  0.5831    &  -    & 0.1524  & 0.1264  \\
  & 6  &    0.4240      &  0.4439    &  0.4501     &  0.6282    &  -    & 0.1704  & 0.1539  \\
  & 7  &    0.4446      &  0.4752    &  0.4826     &  0.6678    &  -    & 0.2064  & 0.1790  \\
  & 8  &    0.4668      &  0.5033    &  0.5103     &  0.7076    &  -    & 0.2212  & 0.2051  \\
  & 9  &    0.4899      &  0.5299    &  0.5380     &  0.7452    &  -    & 0.2540  & 0.2296  \\
  & 10 &    0.5233      &  0.5561    &  0.5658     &  0.7770    &  -    & 0.2931  & 0.2576  \\

    \vspace{1pt}
    \texttt{segmentation} \\

 & 1   &  0.2491    &   0.2546    &   0.2548    &  0.4044   & 0.3969  & 0.0366 & 0.0387  \\
 & 2   &  0.3747    &   0.3775    &   0.3782    &  0.5055   & 0.4934  & 0.0799 & 0.1223  \\
 & 3   &  0.4801    &   0.4864    &   0.4872    &  0.6077   & 0.5800  & 0.1330 & 0.1573  \\
 & 4   &  0.5224    &   0.5340    &   0.5350    &  0.6519   & 0.6236  & 0.2121 & 0.1934  \\
 & 5   &  0.5645    &   0.5827    &   0.5840    &  0.7135   & 0.6746  & 0.2502 & 0.2429  \\
 & 6   &  0.6003    &   0.6318    &   0.6338    &  0.7626   & 0.7464  & 0.3072 & 0.3109  \\
 & 7   &  0.6478    &   0.6657    &   0.6687    &  0.8730   & 0.8374  & 0.3475 & 0.3676  \\
 & 8   &  0.6707    &   0.6877    &   0.6908    &  0.9098   & 0.8635  & 0.4149 & 0.4284  \\
 & 9   &  0.6895    &   0.7091    &   0.7128    &  0.9316   & 0.8904  & 0.4833 & 0.4739  \\
 & 10  &  0.7060    &   0.7341    &   0.7380    &  0.9623   & 0.9121  & 0.5491 & 0.6188  \\

\bottomrule
\end{tabular}
\end{center}
\vspace{6pt}
\end{table}

\newpage

\renewcommand{\arraystretch}{1.5}
\begin{table}[h] \vspace{16pt}
    \caption{Comparison of the variance captured by different dimensionality reduction methods across the maximum dimension $d$}
    \label{tab:compres}
\begin{center}
\Small
\renewcommand{\arraystretch}{1.09}
\begin{tabular}{P{1cm}P{0.3cm}P{1.47cm}P{1.73cm}P{1.5cm}P{1.52cm}P{1.52cm}P{1.52cm}P{1.52cm}}
\toprule

  \textit{Dataset}  & $d$ & \textit{Subset PCA} & \textit{Nyström PCA}  & \textit{Kernel PCA} & \textit{Linear PCA} & \textit{Sparse PCA} & \textit{LLE} & \textit{ICA} \\
    \midrule

    \texttt{drug} \\

  &  1&  0.1338     &  0.1395    &  0.1422     &  0.2316     &  -   &  0.0374    & 0.0278  \\
  &  2&  0.1684     &  0.1787    &  0.1833     &  0.3031     &  -   &  0.0381    & 0.0573  \\
  &  3&  0.2005     &  0.2214    &  0.2279     &  0.3594     &  -   &  0.0919    & 0.0874  \\
  &  4&  0.2256     &  0.2458    &  0.2532     &  0.4060     &  -   &  0.0997    & 0.1149  \\
  &  5&  0.2440     &  0.2728    &  0.2821     &  0.4463     &  -   &  0.1810    & 0.1418  \\
  &  6&  0.2766     &  0.2999    &  0.3098     &  0.4847     &  -   &  0.1936    & 0.1699  \\
  &  7&  0.2962     &  0.3209    &  0.3331     &  0.5087     &  -   &  0.2579    & 0.1905  \\
  &  8&  0.3153     &  0.3478    &  0.3623     &  0.5511     &  -   &  0.2817    & 0.2276  \\
  &  9&  0.3321     &  0.3639    &  0.3796     &  0.5795     &  -   &  0.3198    & 0.2530  \\
  & 10&  0.3474     &  0.3797    &  0.3968     &  0.6037     &  -   &  0.3618    & 0.2766  \\

    \vspace{1pt}
    \texttt{digits} \\

   &  1&   0.0754     &  0.0704     &  0.0721    &  0.0253     &  -   &  0.0109    & 0.0148  \\
   &  2&   0.1299     &  0.1491     &  0.1528    &  0.0455     &  -   &  0.0234    & 0.0289  \\
   &  3&   0.1796     &  0.2057     &  0.2123    &  0.0644     &  -   &  0.0360    & 0.0442  \\
   &  4&   0.2256     &  0.2522     &  0.2614    &  0.0801     &  -   &  0.0492    & 0.0595  \\
   &  5&   0.2733     &  0.2978     &  0.3090    &  0.0951     &  -   &  0.0580    & 0.0772  \\
   &  6&   0.3008     &  0.3258     &  0.3389    &  0.1239     &  -   &  0.0626    & 0.0914  \\
   &  7&   0.3233     &  0.3550     &  0.3716    &  0.1566     &  -   &  0.0805    & 0.1067  \\
   &  8&   0.3494     &  0.3807     &  0.3997    &  0.1791     &  -   &  0.0952    & 0.1221  \\
   &  9&   0.3766     &  0.4054     &  0.4274    &  0.3674     &  -   &  0.1160    & 0.1373  \\
   & 10&   0.3952     &  0.4261     &  0.4498    &  0.3755     &  -   &  0.1289    & 0.1534  \\

    \vspace{1pt}
    \texttt{dailykos} \\

 &  1&       0.0014   &  0.0043    &  0.0047   & 0.0046 & -    &  0.0025  & 0.0033 \\
 &  2&       0.0016   &  0.0100    &  0.0122   & 0.0073 & -    &  0.0046  & 0.0040 \\
 &  3&       0.0023   &  0.0104    &  0.0131   & 0.0080 & -    &  0.0081  & 0.0048 \\
 &  4&       0.0025   &  0.0107    &  0.0138   & 0.0090 & -    &  0.0111  & 0.0054 \\
 &  5&       0.0049   &  0.0111    &  0.0142   & 0.0093 & -    &  0.0139  & 0.0058 \\
 &  6&       0.0054   &  0.0115    &  0.0145   & 0.0097 & -    &  0.0152  & 0.0062 \\
 &  7&       0.0056   &  0.0116    &  0.0148   & 0.0100 & -    &  0.0168  & 0.0065 \\
 &  8&       0.0058   &  0.0120    &  0.0152   & 0.0103 & -    &  0.0170  & 0.0069 \\
 &  9&       0.0062   &  0.0122    &  0.0156   & 0.0108 & -    &  0.0198  & 0.0071 \\
 & 10&       0.0063   &  0.0125    &  0.0160   & 0.0112 & -    &  0.0222  & 0.0074 \\

    \vspace{1pt}
    \texttt{nips} \\

  &  1   &  0.0010  &   0.0026    & 0.0020   &  0.0008     &   -  & 0.0028  & 0.0046  \\
  &  2   &  0.0018  &   0.0034    & 0.0062   &  0.0013     &   -  & 0.0049  & 0.0114  \\
  &  3   &  0.0035  &   0.0043    & 0.0083   &  0.0016     &   -  & 0.0068  & 0.0164  \\
  &  4   &  0.0038  &   0.0050    & 0.0092   &  0.0017     &   -  & 0.0077  & 0.0187  \\
  &  5   &  0.0038  &   0.0067    & 0.0095   &  0.0020     &   -  & 0.0095  & 0.0195  \\
  &  6   &  0.0041  &   0.0095    & 0.0098   &  0.0023     &   -  & 0.0101  & 0.0216  \\
  &  7   &  0.0045  &   0.0098    & 0.0107   &  0.0025     &   -  & 0.0143  & 0.0242  \\
  &  8   &  0.0046  &   0.0100    & 0.0110   &  0.0027     &   -  & 0.0152  & 0.0267  \\
  &  9   &  0.0051  &   0.0102    & 0.0115   &  0.0029     &   -  & 0.0155  & 0.0281  \\
  & 10   &  0.0053  &   0.0105    & 0.0118   &  0.0032     &   -  & 0.0158  & 0.0305  \\

\bottomrule
\end{tabular}
\end{center}
\vspace{6pt}
\end{table}
\emph{}

\normalsize

\emph{}

\newpage

\emph{}

\newpage

To run these experiments using the supplied command-line tool one would do

\mybox{\small{\texttt{> nystrompca methods -\emph{}-seed 1}}}

\vspace{12pt}

Note that the purpose of each of these methods is not necessarily to capture as much variance as possible, however it can still be enlightening to contrast this quantity between different methods. Furthermore, since linear PCA acts in the input space and kernel PCA and its derivations act in the feature space, comparisons of the amount of variance captured are not necessarily clear-cut. Even when linear PCA captures more variance than kernel PCA, the latter may give better performance in downstream tasks, for example when dimensionality reduction is used for preprocessing before carrying out regression or classification.

For all datasets the performance of Nyström kernel PCA is very close to the method it is attempting to approximate, despite being many times more efficient. Nyström kernel PCA also almost always captures more variance than Subset PCA, in particular for the two high-dimensional bag-of-words datasets ({\small \texttt{dailykos}} and {\small \texttt{nips}}). Only for the {\small \texttt{digits}} dataset with a PCA dimension of $1$ is the opposite true. Since we are calculating the reconstruction error on a hold-out dataset it's possible that Subset PCA achieves better performance \--- we know this to be impossible for the training dataset by Proposition \ref{prop:maj}. For datasets with a small number of dimensions standard linear PCA captures the most amount of variance whilst being simpler and more computationally efficient, and so may be the preferred method. The results for sparse PCA are very similar to linear PCA, despite having fewer non-zero entries in the eigenvectors. LLE and ICA generally capture the least amount of variance. LLE is often a good method when the data really lies close to a low-dimensional manifold, which may not be the case for any of the datasets included above. ICA does not attempt to maximize the variance of the new representations and may be preferred when other advantages are being sought. All methods struggle to explain much of the data with only 10 dimensions for the last two datasets that have the highest dimensionality.

Calculation of Nyström kernel PCA takes on average 0.988 seconds across the eight datasets on an AWS EC2 m5.large instance with an Intel Xeon® Platinum 8175M CPU\footnote{\url{https://aws.amazon.com/ec2/instance-types/}} running Ubuntu Server 20.04 with Linux kernel version 5.4, versus 2.753 seconds for full kernel PCA ($n=500, \, m=100$). In both instances the kernel matrices are created in Python whilst the eigendecomposition uses built-in LAPACK routines written in Fortran\footnote{\url{https://numpy.org/devdocs/reference/generated/numpy.linalg.eigh.html}}. For these values of $n$ and $m$ the cubic time complexity is not attained and the constant, linear and quadratic factors are still important.

\vspace{12pt}

\subsection{Bound evaluation} \label{sec:expb} To demonstrate and evaluate the confidence bound as applied to data we compare it to the actual difference between the Nyström reconstruction error and the standard one, as well to the difference between the standard reconstruction error and the reconstruction error for Subset PCA. These quantities are generally not available when applying the Nyström method since they depend on the eigenvalues of the full kernel matrix, but we calculate them here for purposes of illustration.

Here we use a Nyström subset of size $m=50$. We calculate the bound for PCA dimensions $1$ through $10$ and use a confidence level of $0.9$ when calculating the bound. We run the experiments for multiple samples of the Nyström subset and plot the averages for the relevant quantities using $100$ samples. The individual runs for different samples are run in parallel to leverage multi-core CPUs.

We plot the results of the experiments for the first four datasets in Table \ref{tab:datasets} and the kernels in Table \ref{tab:kernels} for different PCA dimensions below in Figures 1, 2 and 3. For the RBF and Cauchy kernels we set the bandwidth to $\sigma = 1$ and for the polynomial kernel we use $R = 1$ and $d = 2$. The RBF and Cauchy kernels are bounded by $\sup_x k(x,x) = 1$ and we normalize the polynomial kernel according to Equation~(\ref{eq:knorm}) before applying it in the experiments. Each plot contains
\begin{enumerate}
    \item The values of the confidence bound (``Conf. bound'')
    \item The difference between the Nyström PCA and standard errors $R_n(\widetilde{V}_d) - R_n(\hat{V}_d^n)$ \newline (``Nyström diff.'')
    \item The difference between the Subset PCA and standard errors $R_n(\hat{V}_d^m) - R_n(\hat{V}_d^n)$ \newline (``Subset diff.'')
\end{enumerate}

Both the Nyström difference, the subset difference and the bound increase as the PCA dimension increases. The bound increases more rapidly as the PCA dimension increases from low values, but levels out for larger values as the tail eigenvalues decrease.

The bound seems fairly conservative for these datasets and these choices of hyperparameters. In real-life applications of the Nyström method the datasets are usually much larger, with the number of data points sometimes in the millions, and then the bound will be significantly smaller. The main purpose of the current experiments is rather to investigate differences between datasets and kernel functions and across PCA dimensions. %The same experiments could not be performed for actual datasets where the Nyström method is selected to make kernel PCA scalable, since in this case the calculation of the eigendecomposition is intractable.
%%% FIGURES %%%

\vspace{12pt}

\begin{figure}[H] \label{fig:rbf}
    \hspace*{0.65in}
    \includegraphics[width=4.9in]{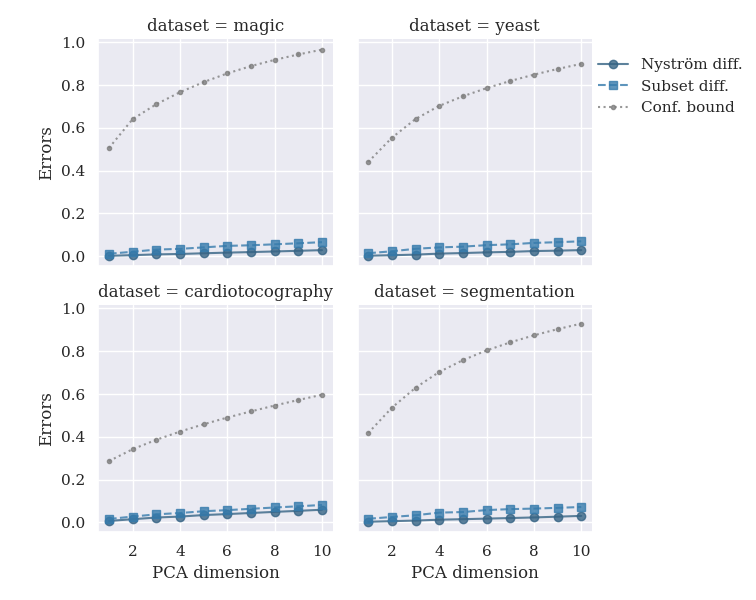}
    \caption{Error comparison with the RBF kernel}
\end{figure}

\emph{}

\begin{figure}[H] \label{fig:poly}
    \hspace*{0.65in}
    \includegraphics[width=4.9in]{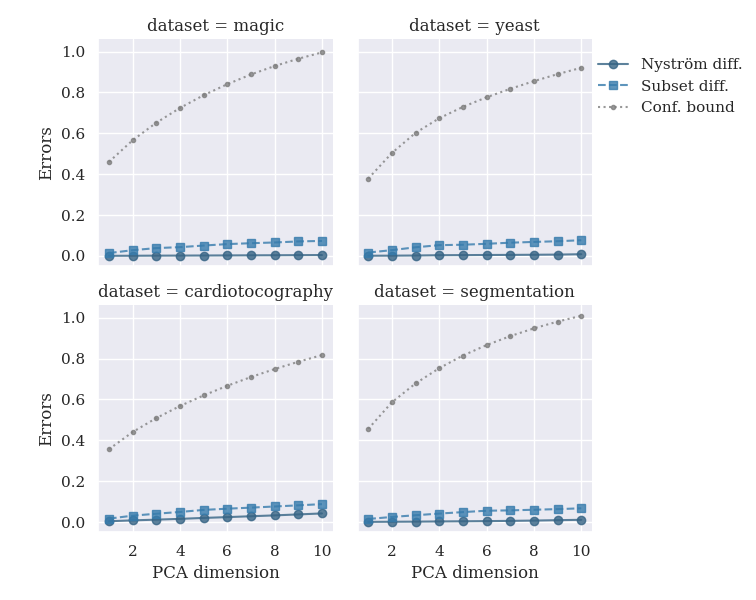}
    \caption{Error comparison with the polynomial kernel}
\end{figure}
~
\begin{figure}[H] \label{fig:cauchy}
    \hspace*{0.65in}
    \includegraphics[width=4.9in]{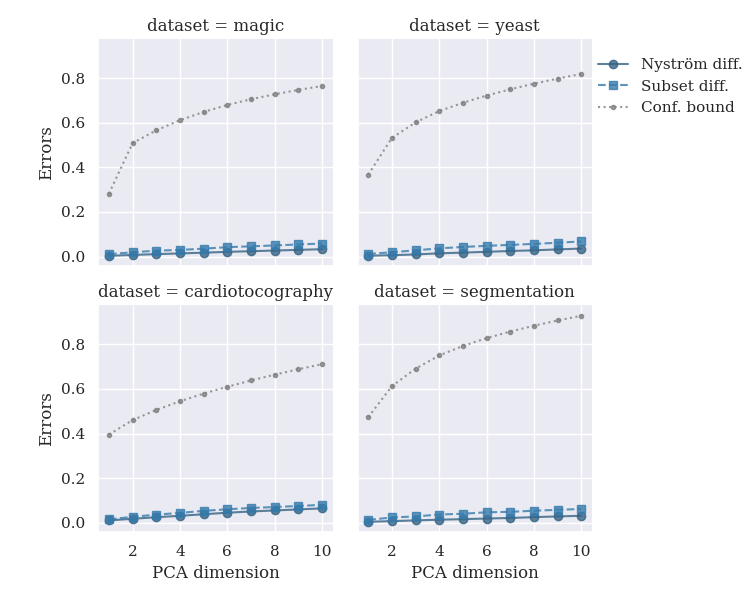}
    \caption{Error comparison with the Cauchy kernel}
\end{figure}

Running the bound evaluation experiments with the command-line tool can be accomplished with the following command%, which uses default values for all parameters

\mybox{\small{\texttt{> nystrompca bound -m 50}}}

\vspace{6pt}

In this section we have presented the experimental results for a few select parameter values. Other combinations can easily be tried after downloading the supplied software package.

% Length before/after equations
\setlength{\abovedisplayskip}{16pt}
\setlength{\belowdisplayskip}{20pt}

\vspace{16pt}

\section{Application: Nyström principal component regression} \label{sec:appl}

\vspace{12pt}

As an application of Nyström kernel PCA we present kernel principal component regression with the Nyström method, or \emph{Nyström kernel PCR}. The proposed method may be used for efficient regularized kernel regression, for example as an alternative to kernel ridge regression with the Nyström method \citep{liang2020just}. Its derivation demonstrates how the principal scores from Nyström kernel PCA may be used as new data points for supervised learning methods.

Principal component regression performs a regression of a target variable onto the principal scores from a subset of the principal components, instead of using the original data as regressor variables \citep[Chapter~8]{jolliffe2002principal}. Principal component regression introduces regularization and ameliorates collinearity of the regressors, which leads to high variances for the coefficient estimates and may especially be a problem for kernel methods. It is known to correspond to the \emph{errors-in-variables} regression model under certain circumstances, where the dependent and independent variables are assumed to contain measurement noise \citep{fuller1980properties}.

We first derive standard kernel PCR, without the Nyström method. This derivation appears to be novel, as previous presentations of kernel principal component regression assumed data to have zero mean in feature space \citep{rosipal2000kernel, rosipal2001kernel, wibowo2012note}.

Suppose thus that each data point $x_i$ is paired with an observation of a target variable $y_i$ in $\R$ which we wish to predict using a new observation $x^*$ of the independent variable. The regression model is
\[
    y =  \alpha + S_d \beta + \varepsilon
\]
with parameters $\alpha$ and $\beta = (\beta_1, \beta_2, ..., \beta_d)^T$, where $y = (y_1, y_2, ..., y_n)^T$, $S_d$ are the principal scores from kernel PCA with respect to the top $d$ principal components,  and $\varepsilon$ is a noise vector $\varepsilon = (\varepsilon_1, \varepsilon_2, ..., \varepsilon_n)^T$, whose components we assume are generated from a zero-mean distribution with finite variance $\mathrm{Var}(\varepsilon_i)$. The intercept is given by $\alpha = \bar{y}$ since the scores have zero mean in each dimension. From Section \ref{sec:background} the principal scores are given by $S_d = Q_d \Lambda^{1/2}_d$, where $Q_d \Lambda_d Q_d^T$ is the truncated eigendecomposition of $K'$. Since we assumed $Z$ to be square-integrable we may apply least squares estimation to obtain that \citep{sen2010finite}
\[
    \hat{\beta}
    = (S_d^TS_d)^{-1} S_d^T y'
    = \Lambda_d^{-1/2} Q_d^T y'
    = S_d^{-1} y'
\]
where $y' = (y_1 - \bar{y}, \, y_2 - \bar{y}, \, ..., \, y_n - \bar{y})^T$. We recall that the principal scores of a new data point $x^*$, which we centre since we estimated the regression for zero-mean data points, are given by, with respect to the top $d$ principal components
\[
    w^*_d
    = \Lambda_d^{-1/2}Q_d^T \kappa'(x^*)
    = \Lambda_d^{-1/2}Q_d^T\left(
        \kappa(x^*) - \one_n \kappa(x^*)  - K \mathbf{1}_n + \one_n K \mathbf{1}_n
    \right)
\]
and so the prediction for a new data point becomes
\[
    \hat{y} = \bar{y} + \beta^T w^{* T}_d = \bar{y} +  y^{\prime \, T} Q_d \Lambda_d^{-1} Q_d^T \kappa'(x^*)
\]
For the Nyström method, the principal scores are given by $W = K^{\prime}_{nm}K_{mm}^{\prime-1/2} V = K'_{nm} U$, and so the principal scores with respect to the top $d$ principal components are given by $W_d = K'_{nm}K_{mm}^{\prime-1/2} V_d = K^{\prime}_{nm} U_d$ where $V_d\widetilde{\Lambda}_dV_d^T$ is the truncated eigendecomposition of $\frac{1}{n}K_{mm}^{\prime -1/2} K'_{mn} K'_{nm} K_{mm}^{\prime-1/2}$ and $U_d = K_{mm}^{\prime \, -1/2} V_d$. The regression model then becomes
\[
    y =  \alpha + W_d \beta + \varepsilon
      =  \alpha + K'_{nm} U_d \beta + \varepsilon
      = \alpha + K'_{nm}K_{mm}^{\prime-1/2} V_d \beta + \varepsilon
\]
The least squares parameter estimates are $\hat{\alpha} = \bar{y}$ and
\begin{align*}
    \hat{\beta}
    &= (W_d^TW_d)^{-1} W_d^T y'
    = \left( V_d^T K_{mm}^{\prime-1/2} K'_{mn} K'_{nm}K_{mm}^{\prime-1/2} V_d\right)^{-1}
      V_d^T K_{mm}^{\prime-1/2} K'_{mn} y' \\[1em]
    &= \left( (V_d^T V \widetilde{\Lambda} V^T  V_d\right)^{-1} V_d^T K_{mm}^{\prime-1/2} K'_{mn} y'
    = \widetilde{\Lambda}_d^{-1} V_d^T K_{mm}^{\prime-1/2} K'_{mn} y'
    = \widetilde{\Lambda}_d^{-1} U_d^T K'_{mn} y'
\end{align*}
And so the prediction becomes
\[
    \hat{y} =
    \bar{y} +
    y^{\prime \, T} K'_{nm} U_d \widetilde{\Lambda}^{-1}_d U_d^T \, \widetilde{\kappa}(x^*)
\]
We implement kernel principal component regression with the Nyström method (Nyström KPCR) in computer experiments and compare it with Nyström kernel ridge regression (Nyström KRR) \citep{rudi2015less}, which is given by\footnote{This is a slightly different specification than in \cite{rudi2015less}, where we have demeaned the target variable and subsumed a factor $n$ into the ridge parameter}
\begin{align*}
    \hat{y} &= \bar{y} + \beta^T \kappa(x^*) \\[0.5em]
    \hat{\beta}   &= (K_{mn}K_{nm} + \gamma K_{mm})^{-1} K_{mn}y'
\end{align*}
where $\gamma \ge 0$ is a regularization parameter, called the ridge parameter.

We use the \texttt{airfoil} dataset from the UCI machine learning repository \citep{dua2019uci}, which describes aerodynamic tests of blades in a wind tunnel from NASA and contains $1503$ data points and $6$ attributes. 

Again we normalize the attributes to have mean $0$ and variance $1$. Note that we must not normalize the entire dataset at once so as to not introduce look-ahead bias in the regression \--- when creating a prediction for a new data point we need to normalize this data point using the mean and variance from the training set.

For this experiment we use the radial basis functions kernel with parameter $\sigma=1$. The source code for these experiments is available in the same package at {\small\url{https://github.com/fredhallgren/nystrompca}}. We estimate the regression on a training dataset with a random sample of 75 \% of all data points, and evaluate the method on a test set with the remaining data points.

We first plot the $R^2$ for the regression on the test set for different subset sizes $m$, ridge parameters $\gamma$ and PCA dimensions $d$ below in Figure 4. We do not specify a random seed for these plots and for each parameter combination a different subset is used.

For Nyström kernel PCR the regression accuracy improves as we increase the number of principal components used in the regression and as the size of the subset increases. For Nyström KRR the accuracy also improves with a larger subset, but the pattern is less clear as we change the regularization parameter.

%%% FIGURE %%%
\begin{figure}[H] \centering
    \begin{subfigure}[b]{0.49\textwidth}
        \includegraphics[width=\textwidth]{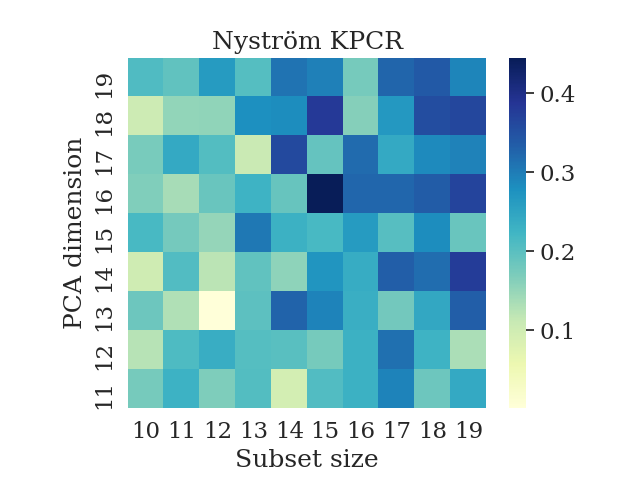}
    \end{subfigure}
    \begin{subfigure}[b]{0.49\textwidth}
        \includegraphics[width=\textwidth]{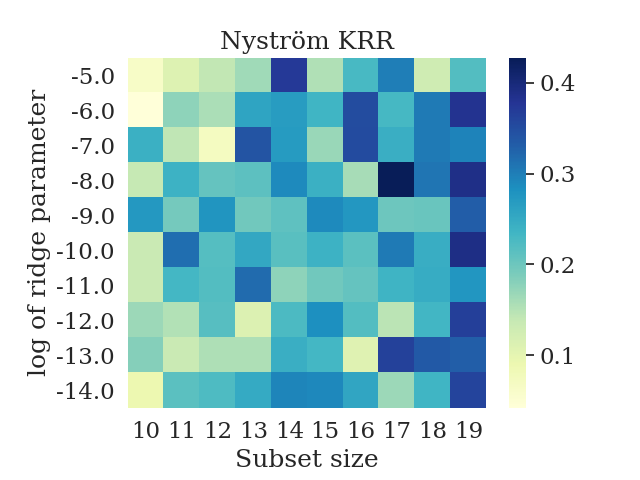}
    \end{subfigure}
    \caption{Heat maps with regression $R^2$}
\end{figure}

To further elucidate the behaviour of the methods we also plot the actual target values versus the predicted ones on the test set for one instance of the parameters. Please see below Figure 5. Here we use $m = 100$, $d = 90$ and $\gamma = 10^{-11}$. The parameters $d$ and $\gamma$ were manually tuned. In this particular example Nyström KPCR obtained an $R^2$ of $0.74$ and Nyström KRR $0.72$ with a seed of $1$.

The scatter plots of the predictions versus the actual targets look as expected for an $R^2$ of around $0.7$. The predictions for the two methods look quite similar, but slightly different characteristics are exhibited by the plots due to the different regularization methodologies.

%%% FIGURE 2 %%%

\begin{figure}[H] \centering
    \begin{subfigure}[b]{0.49\textwidth}
        \includegraphics[width=\textwidth]{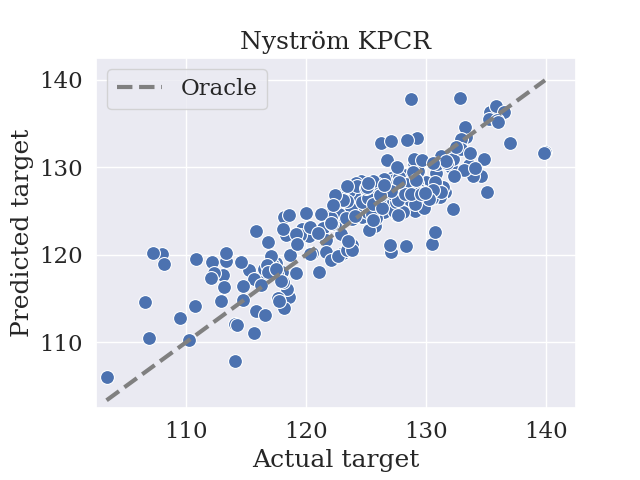}
    \end{subfigure}
    \begin{subfigure}[b]{0.49\textwidth}
        \includegraphics[width=\textwidth]{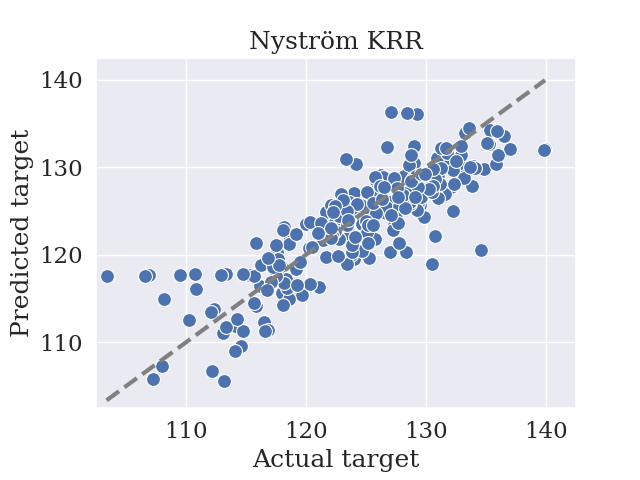}
    \end{subfigure}
    \caption{Scatter plot with regression predictions}
\end{figure}

\newpage
These experiments can also be run with the command-line tool, using the below command
\mybox{\small{\texttt{> nystrompca regression -m 100 -d 90}}}

\vspace{6pt}

where the \texttt{-m} flag changes the size of the subset for the second set of plots, and \texttt{-d} selects the number of PCA dimensions for Nyström kernel PCR in the same plots.  Please see the software documentation for a full list of available command-line options.%, for example to set different values for the ridge parameter.

\vspace{24pt}

\section{Conclusion} \label{sec:conc}

\vspace{12pt}

In this paper we have presented an efficient method for non-linear PCA by deriving the Nyström method for kernel PCA, providing the principal components, explained variance, principal scores and reconstruction error, without assuming that data has zero mean. The Nyström method has been shown in theory and practice to be effective for improving the scalability of kernel methods, but it had not yet been derived for kernel PCA in line with classical PCA. The method presented in this paper is one of the few options available to reduce the computational requirements for kernel PCA while maintaining good accuracy.

We also provided a finite-sample confidence bound on the empirical reconstruction error of the method, which allows us to measure its statistical accuracy for a specific dataset and which can be applied before the entire dataset has been observed. As a corollary to the confidence bound we presented a brief lemma by which a large number of previous results can be sharpened, and of potentially wide application in the future.
%The bound assumes data has zero mean in feature space, but could potentially be adapted to account for centring of data points, although the analysis would become more involved.

The principal scores from the method may be used instead of the original data matrix in any supervised learning method, in order for example to achieve regularization and denoising. We demonstrated this for linear regression by presenting Nyström kernel principal component regression. This derivation also led to a novel specification for standard kernel PCR, where the regressors are centred, as is the case for linear PCR.

We thank you for your interest in this work. We hope that what we have presented in this paper will be useful to the academic community and to industry practitioners, and that it may give ideas about future directions of research.

In addition to linear regression, there are many other methods based on PCA where kernel PCA with the Nyström method could be analyzed and explored, such as when PCA is applied in discriminant analysis, outlier detection or dictionary learning. The latter could be achieved for example along the lines of \cite{golts2016linearized}.

The approximate Nyström kernel matrix $\widetilde{K} = K_{nm}K_{mm}^{-1}K_{mn}$ may often be used as a drop-in replacement for the original kernel matrix to speed up kernel machines. However, for many methods, like kernel PCA, more work is needed for a complete treatment. There are still many kernel methods where application of the Nyström method is not necessarily trivial and has not been fully derived, including potentially kernel FDA or kernel PLS \citep{mika1999fisher, rosipal2001kernelb}

Kernel PCA is closely related to functional PCA \citep{amini2012sampled}. Functional PCA may also suffer from scalability issues if the individual functions are sampled at a large number of points. It's possible that there are settings where the Nyström method could be successfully applied to functional data analysis for improved computational efficiency. 

As outlined in Section~\ref{sec:prevwork} on page~\pageref{sec:prevwork}, since the one existing bound for centred kernel PCA is more conservative than the available bounds for uncentred kernel PCA, there no statistical analysis of kernel PCA available in its full generality that is as precise as with an assumption of zero-mean data. Closing this gap would be a major achievement and in our view one of the most significant potential research contributions for kernel methods in the near term.

In this paper the data points in the Nyström subset are selected randomly from the full dataset to provide a subspace in which the PCA solution is sought, but the selection of these data points is not optimized in any way to best suit the problem at hand. In the Gaussian process literature, the so-called \emph{inducing points}, which are analogous to the Nyström subset, are often selected to be optimal in some sense \citep{wild2021connections}. Other ways to select the Nyström subset for our method could also be explored, which may achieve some measure of optimality, or benefit from improved performance. 

In an ideal world one would not randomly constrain some subspace in pursuit of improved computational performance, a crude proxy without doubt, but instead optimize with respect to the computational resources directly, where some measure of computational complexity is seamlessly included in the estimation procedure. Sadly, such theory does not yet exist. The theory of statistical estimation is conspicuously detached from that of computation, so given some utility of improved statistical accuracy, and some model of computation, there is no way to objectively, or even formally, trade accuracy against computational cost. Until such time, we are reduced to considering a plethora of different models, each with its own separate measures of accuracy and complexity, but with little recourse when attempting to choose between them.

\newpage

\appendix

\newgeometry{margin=1in, top=1in, bottom=1.3in, headsep=0in, footskip=0.6in}

\fancyhf{}
\fancyfoot[C]{{\Small \thepage}}
\fancyfoot[LR]{}

\section{Proofs}

\vspace{16pt}

In this section we present the proofs of Propositions \ref{prop:maj} and \ref{prop:trueeq}, and Theorems \ref{thm:nystrompca}, \ref{thm:subsetpca} and \ref{thm:prob}.

First we state and prove a lemma that is used in the proof of Theorem~\ref{thm:prob}. It shows that the bound of the difference between two positive operators is smaller than the sum of the bounds for the individual operators by a factor $1/\sqrt{2}$. For wider applicability we state this lemma for Hilbert spaces over both the real or complex numbers.

\begin{lemma}[Bound of positive operators] \label{lem:opbound}
Let $L_1$ and $L_2$ be positive operators in $\hs$  over $\R$ or $\C$ with $\|L_1\|_{\hs} \le B$ and $\|L_2\|_{\hs} \le B$. Then $\|L_1 - L_2\|_{\hs} \le \sqrt{2}B$.	
\end{lemma}

\begin{proof}
$\|L_1 - L_2\|^2_{\hs} = \|L_1\|^2_{\hs} + \|L_1\|^2_{\hs} - 2\mathrm{Re}\{\langle L_1, L_2 \rangle_{\hs}\} \le 2B^2$ since $\langle L_1, L_2 \rangle_{\hs}$ is real and positive. To see this, we note that $\langle L_1, L_2 \rangle_{\hs} = \sum_{i=1}^{\infty} \langle L_1 e_i, L_2 e_i \rangle_{\Hcal}$ for any basis $\{e_i\}$ in $\Hcal$ and take $\{e_i\}$ to be the eigenvectors of $L_1$, arbitrarily extended to a basis for the entire space if $\mathrm{Ker}(L_1) \neq \{0\}$, obtaining, for each $i$, that $\langle L_1 e_i, L_2 e_i \rangle_{\Hcal} = \langle \lambda_i e_i, L_2 e_i \rangle_{\Hcal} = \lambda_i \langle  e_i, L_2 e_i \rangle_{\Hcal} \in \R_+$ since $L_2$ is positive and $\lambda_i \in \R_+$.
\end{proof}

%\newpage

\begin{proof}[\textbf{Proof of Theorem \ref{thm:nystrompca}}]

Standard principal component analysis finds the perpendicular intersecting lines in $\R^d$ along which the variance of the data is successively maximized \citep{pearson1901principal}. These lines are affine subspaces of $\R^d$ which are orthogonal with respect to the associated vector space. To derive kernel PCA with the Nyström method we apply PCA in the span of the subset of data points $\Hcal_S$, i.e. finding the orthogonal one-dimensional affine subspaces of $\Hcal_S$ where the projected data has maximum variance. These are on the form
\[
    \phi_0 + \langle f_j \rangle
  = \phi_0 + \{ \, a f_j \, | \, a \in \R \,  \}
\]
where $\phi_0 \in \Hcal_S$ is the translation of the vector space $\langle f_j \rangle$, and the $f_j \in \Hcal_S$, taken to have norm one, are the principal components. It is known from standard PCA and kernel PCA that the translation vector is given by the mean of the data points, which in our case is the mean of the data points projected onto $\Hcal_S$. Using $P_{\Hcal_S} = G_m^*(G_mG_m^*)^{-1}G_m = m \cdot G_m^*K_{mm}^{-1}G_m$, where $G_m$ is the sampling operator \citep{rudi2015less}, we obtain
\begin{align*}
\phi_0  
&= \frac{1}{n} \sum_{r=1}^n P_{\Hcal_S}\phi(x_r)
= \frac{1}{n} \sum_{r=1}^n m \cdot G_m^*K_{mm}^{-1} G^{}_m \phi(x_r) \\[1em]
&= \frac{1}{n} \sum_{r=1}^n \sqrt{m} \cdot G_m^*K_{mm}^{-1} \kappa_m(x_r) = \frac{1}{n}K_{nm}K_{mm}^{-1} \kappa_m(x)
\end{align*}

Any element $\phi \in \Hcal_S$ can be written as $\phi = \phi_0 + \sum_{k=1}^m a_k \cdot (\phi(x_k) - \phi_0)$ for some coefficients $a_1, a_2,$\\$ ..., a_m$ and so the principal components are on the form $f_j = \sum_{k=1}^m u_{j,k}(\phi(x_k) - \phi_0)$ with coefficients $u_{j,1}, u_{j,2}, ..., u_{j,m}$. The affine projection of a data point $\phi(x)$ onto $\phi_0 + \langle f_j \rangle$ is then
\[
        P_{\phi_0 + \langle f_j \rangle} \phi(x) = \phi_0 + \langle \phi(x) - \phi_0, f_j  \rangle_{\Hcal} f_j \\[1em]
\]
The variance of the full dataset along $\phi_0 + \langle f_j \rangle$ then becomes
\begin{align*}
    &\mathrm{Var}_{f_j} \left(\{ \phi(x_i) \}_{i=1}^n\right)
    = \frac{1}{n} \sum_{i=1}^n
    \left(
    \phi_0 + \langle \phi(x_i) - \phi_0, f_j \rangle_{\Hcal}
    -
    \frac{1}{n}
    \sum_{\ell=1}^n
    \left( \phi_0 +  \langle \phi(x_{\ell}) - \phi_0, f_j \rangle_{\Hcal} \right)
    \right)^2 \\[1em]
    &=
    \frac{1}{n} \sum_{i=1}^n
    \left\langle
    \phi(x_i) - \frac{1}{n}\sum_{\ell=1}^n \phi(x_{\ell}),
    \sum_{k=1}^m u_{j,k} \left(\phi(x_k) - \phi_0 \right)
    \right\rangle_{\Hcal}^2
    \\[1em]
    &=\frac{1}{n} \sum_{i=1}^n \left( \sum_{k=1}^m u_{j,k}
        \left(
        k_{k,i}
        - \frac{1}{n}\sum_{\ell=1}^n k_{k,\ell}
        - \langle \phi(x_i), \phi_0 \rangle_{\Hcal}
        + \frac{1}{n}\sum_{\ell=1}^n \langle \phi(x_{\ell}), \phi_0 \rangle_{\Hcal}
    \right) \right)^2 \\[1em]
\end{align*}
Using 
\begin{align*}
    &\langle \phi(x_i), P_{\Hcal_S} \phi(x_r) \rangle_{\Hcal}
    = \langle \phi(x_i), m \cdot G_m^*K_{mm}^{-1}G_m \phi(x_r) \rangle_{\Hcal} \\[1em]
    &= \sqrt{m}\langle \phi(x_i), G_m^*K_{mm}^{-1} \kappa_m(x_r) \rangle_{\Hcal}
    = \kappa_m(x_i)^T K_{mm}^{-1} \kappa_m(x_r)
\end{align*}
where $\kappa_m(x) = (k(x_1, x), \; k(x_2, x), \; ..., \; k(x_m, x))^T$, and setting $\kappa_m(x_a) = \kappa_{m,a}$, we obtain
\begin{align*}
    &\frac{1}{n} \sum_{i=1}^n \left( \sum_{k=1}^m u_{j,k}
        \left(
        k_{k,i}
        - \frac{1}{n}\sum_{\ell=1}^n k_{k,\ell}
        - \langle \phi(x_i), \phi_0 \rangle_{\Hcal}
        + \frac{1}{n}\sum_{\ell=1}^n \langle \phi(x_{\ell}), \phi_0 \rangle_{\Hcal}
    \right) \right)^2 \\[1em]
    &=\frac{1}{n} \sum_{i=1}^n \left( \sum_{k=1}^m u_{j,k}
        \left(
        k_{k,i}
        - \frac{1}{n}\sum_{\ell=1}^n k_{k,\ell}
        - \frac{1}{n}\sum_{r=1}^n \kappa_{m,i}^T K_{mm}^{-1} \kappa_{m,r}
        + \frac{1}{n^2}\sum_{\substack{\ell=1 \\ r=1}}^n \kappa_{m,\ell}^T K_{mm}^{-1} \kappa_{m,r}
    \right) \right)^2 \\[1em]
    &= \frac{1}{n} u_j^T K'_{mn}K'_{nm} u_j
\end{align*}
where $K'_{mn} = 
       K_{mn} - K_{mn}\one_n 
        - \one_n^{m, n}K_{nm} K_{mm}^{-1} K_{mn}
        + \one_n^{m, n} K_{nm} K_{m,m}^{-1} K_{mn} \one_n
$, with $\one_n^{m, n}$ an $m \times n$ matrix with each element equal to $\frac{1}{n}$, and $K'_{nm} = K^{\prime \, T}_{mn}$.

The principal components are then given by the orthonormal vectors $f_j = \sum_{k=1}^m u_{j,k}(\phi(x_k) - \phi_0)$, $j = 1, 2, ..., m$ that successively maximize the variance. The inner product between two principal components is

\begin{align*}
    &\langle f_j, f_p \rangle_{\Hcal}
    = \left \langle \sum_{k=1}^m u_{j,k} \left(\phi(x_k) - \phi_0 \right), \,
       \sum_{q=1}^m u_{p,q} \left(\phi(x_q) - \phi_0 \right)\right \rangle_{\Hcal} \\[1em]
    &= \sum_{\substack{k=1 \\ q=1}}^m u_{j,k} u_{p,q} \left(
            k_{k,q}
            - \frac{1}{n} \sum_{r=1}^n \kappa_{m,r} K^{-1}_{mm} \kappa_{m,k}
            - \frac{1}{n} \sum_{\ell=1}^n \kappa_{m,\ell} K^{-1}_{mm} \kappa_{m,q}
            + \frac{1}{n^2} \sum_{\substack{r=1 \\ \ell=1}}^n \kappa_{m,r} K^{-1}_{mm} \kappa_{m,\ell}
       \right) \\[1em]
    &= u_j^T K'_{mm} u_p
\end{align*}
where $K'_{mm} = K_{mm} - \one_n^{m, n} K_{nm} - K_{mn} \one_n^{n, m} + \one_n^{m, n} K_{nm} K^{-1}_{mm} K_{mn} \one_n^{m, n}$. Maximizing the variance therefore becomes a generalized eigenvalue problem. We have
\[
    \langle f_j, f_p \rangle_{\Hcal}
    = u_j^T K'_{mm} u_p
    = \left( K_{mm}^{\prime \, 1/2} u_j \right)^T \left( K_{mm}^{\prime \, 1/2} u_p \right) := v_j^Tv_p
\]
where $K_{mm}^{\prime \, 1/2}$ is the unique positive semi-definite square root of $K'_{mm}$ given by $m \cdot U^m \Lambda^{m \, 1/2} U^{m \, T}$, where $U^m \Lambda^m U^{m \, T}$ is the eigendecomposition of $\frac{1}{m}K_{mm}'$. Therefore the variance can be written
\[
    \frac{1}{n} v_j^T K_{mm}^{\prime \, -1/2} K'_{mn} K'_{nm} K_{mm}^{\prime \, -1/2} v_j
    = \left \langle v_j, \; \frac{1}{n} K_{mm}^{\prime \, -1/2} K'_{mn} K'_{nm} K_{mm}^{\prime \, -1/2} v_j \right \rangle_{\R^m}
\]
Then by the Courant-Fischer-Weyl theorem \citep[Corollary~III.1.2]{bhatia1997matrix} the maximum values over successively orthonormal vectors $v_j$ are given by the eigenvalues of $\frac{1}{n} K_{mm}^{\prime \, -1/2} K'_{mn} K'_{nm} K_{mm}^{\prime \, -1/2}$, and they occur at its eigenvectors. These eigenvectors will be unique (up to a sign), since all data points are different by assumption.

The principal components are then given by
\[
    \widetilde{\phi}_j =  \sum_{k=1}^m u_{j,k} \left( \phi(x_k) - \phi_0 \right) \; \; \;  \; \; \; \; \; \; j = 1, 2, ..., m
\]
where $u_j = K_{mm}^{\prime \, -1/2} v_j$, and the affine subspaces with maximum variances are $\{ \, \phi_0 + t \widetilde{\phi}_j \; | \; t \in \R \, \}$, $j = 1, 2, ..., m$.

The principal score of a centred data point $i$ with respect to the principal component $j$ is given by
\begin{align*}
    w_{j,i}
    &= \left\langle \phi(x_i) - \frac{1}{n}\sum_{\ell=1}^n\phi(x_{\ell}), \,
       \sum_{k=1}^m u_{j,k}(\phi(x_k) - \phi_0) \right\rangle_{\Hcal} \\[1em]
    &=\sum_{k=1}^m u_{j,k} \left(
        k_{k, i}
        - \frac{1}{n}\sum_{\ell=1}^n k_{k, \ell}
        - \frac{1}{n}\sum_{r=1}^n \kappa_{m,i}^T K_{mm}^{-1} \kappa_{m,r}
        + \frac{1}{n^2}\sum_{\substack{\ell=1 \\ r=1}}^n \kappa_{m,\ell}^T K_{mm}^{-1} \kappa_{m,r}
    \right)
\end{align*}
for $j = 1, 2, ..., n$. Or in matrix format
\[
    (w_{i,j}) = W = K'_{nm} U
\]
where $U = K_{mm}^{\prime \, -1/2} V$ and $\frac{1}{n}K_{mm}^{\prime \, -1/2} K'_{mn} K'_{nm} K_{mm}^{\prime \, -1/2} = V \widetilde{\Lambda} V^T$, and so $W =  K'_{nm} K_{mm}^{\prime \, -1/2} V$. 

The scores of a \emph{new} data point $x^*$ which is centred in feature space, i.e. the coordinates of $\phi(x^*) - \frac{1}{n}\sum_{\ell=1}^n \phi(x_{\ell})$ in terms of the principal components, are given by
\begin{align*}
    w_j^*
    &= \left\langle
          \phi(x^*) - \frac{1}{n}\sum_{\ell=1}^n\phi(x_{\ell}), \;
          \sum_{k=1}^m u_{j,k} \left( \phi(x_k) - \phi_0 \right)
      \right\rangle_{\Hcal}
    \\[1em]
    &= \sum_{k=1}^m u_{j,k} \left(
        k(x_k, x^*)
        - \frac{1}{n}\sum_{\ell=1}^n k_{k,\ell}
        - \frac{1}{n}\sum_{r=1}^n \kappa_{m,r}^T K^{-1}_{mm} \kappa_m(x^*)
        + \frac{1}{n^2}\sum_{\substack{r=1 \\ \ell=1}}^n \kappa_{m,r}^T K^{-1}_{mm} \kappa_{m,\ell}
    \right)
\end{align*}
or in matrix format
\begin{align*}
    w^*
    = U^T \left(
        \kappa_m(x^*) - K_{mn} \mathbf{1}_n - \one_n^{m, n}K_{nm}K^{-1}_{mm} \kappa_m(x^*)  + \one_n^{m, n}K_{nm}K^{-1}_{mm}K_{mn} \mathbf{1}_n
    \right)
    := U^T \, \widetilde{\kappa}(x^*)
\end{align*}
where $\mathbf{1}_n$ is a length-$n$ column vector given by $\mathbf{1}_n = (\frac{1}{n}, \frac{1}{n}, \; ..., \; \frac{1}{n})^T$.

\end{proof}

\begin{proof}[\textbf{Proof of Theorem \ref{thm:subsetpca}}]

The projection of a data point $\phi(x_i)$ onto a principal component is given by
\begin{align*}
    P_{\hat{\phi}_j^{m,n}}\phi(x_i)
    &= \frac{1}{\sqrt{m \hat{\lambda}_j^m}} \sum_{k=1}^m u^m_{j,k}
    \left\langle \phi(x_i), \, \phi(x_k) - \phi_0 \right\rangle_{\Hcal}  \hat{\phi}_j^{m,n} \\[1em]
    &= \frac{1}{\sqrt{m \hat{\lambda}_j^m}} \sum_{k=1}^m u^m_{j,k}
    \left( k(x_k, x_i) - \frac{1}{n} K_{nm} K^{-1}_{mm}  \kappa_m (x_i) \right) \hat{\phi}_j^{m,n}
\end{align*}
where $(\hat{\lambda}_j^m, u^m_j)$ is the $j$th eigenpair of $\frac{1}{m}K'_{mm}$ and $u^m_{j,k}$ is the $k$th element of $u^m_{j}$ \citep{shawe2005eigenspectrum}.

The projection of a centred data point $\phi'(x_i)$ is then, similarly to Theorem~\ref{thm:nystrompca}, with $k_{a,b} := k(x_a, x_b)$ and $\kappa_m(x_a) = \kappa_{m,a}$

\begin{align*}
    &P_{\hat{\phi}_j^{m,n}}\phi'(x_i)
    = \frac{1}{\sqrt{m \hat{\lambda}_j^m}} \sum_{k=1}^m u^m_{j,k}
    \left\langle \phi(x_i)
                - \frac{1}{n} \sum_{\ell=1}^n \phi(x_{\ell}) , \,
                \phi(x_k)
                - \phi_0
    \right\rangle_{\Hcal} \hat{\phi}_j^{m,n} = \\[1em]
\end{align*}
\begin{align*}
    &= \frac{1}{\sqrt{m \hat{\lambda}_j^m}} \sum_{k=1}^m u^m_{j,k}
    \left(
        k_{k,i}
        - \frac{1}{n}\sum_{\ell=1}^n k_{k,\ell}
        - \frac{1}{n}\sum_{r=1}^n \kappa_{m,i}^T K_{mm}^{-1} \kappa_{m,r}
        + \frac{1}{n^2}\sum_{\substack{\ell=1 \\ r=1}}^n \kappa_{m,\ell}^T K_{mm}^{-1} \kappa_{m,r}
    \right) \hat{\phi}_j^{m,n}
\end{align*}
Taking the norm and summing over $\phi(x_1), \phi(x_2), ..., \phi(x_n)$ we obtain
\begin{align*}
    &\frac{1}{n} \sum_{i=1}^n \| P_{\hat{\phi}_j^{m,n}}\phi'(x_i) \|_{\Hcal}^2 = \\[1em]
    &\frac{1}{n \cdot m\hat{\lambda}_j^m} \sum_{i=1}^n
    \left(
     \sum_{k=1}^m u^m_{j,k}
    \left(
        k_{k,i}
        - \frac{1}{n}\sum_{\ell=1}^n k_{k,\ell}
        - \frac{1}{n}\sum_{r=1}^n \kappa_{m,i}^T K_{mm}^{-1} \kappa_{m,r}
        + \frac{1}{n^2}\sum_{\substack{\ell=1 \\ r=1}}^n \kappa_{m,\ell}^T K_{mm}^{-1} \kappa_{m,r}
    \right)
    \right)^2 \\[1em]
    &= \frac{1}{n \cdot m\hat{\lambda}_j^m} u^{m \, T}_j K'_{mn}K'_{nm} u^m_j
    \, =: \, \hat{\lambda}_j^{m,n}
\end{align*}
For the reconstruction error we have
\begin{align*}
R_n(\hat{V}_d^m)
&= \frac{1}{n} \sum_{i=1}^n
  \|
    \phi'(x_i) - P_{\hat{V}_d^m} \phi'(x_i)
  \|^2_{\Hcal} 
  =
    \frac{1}{n} \sum_{i=1}^n
    \| \phi'(x_i) \|_{\Hcal} -
    \frac{1}{n} \sum_{i=1}^n
    \left\|
    P_{\hat{V}_d^m} \phi'(x_i)
    \right\|_{\Hcal} \\[1em]
&=
    \frac{1}{n} \tr(K') -
    \frac{1}{n} \sum_{i=1}^n
    \left\|
    P_{\hat{V}_d^m} \phi'(x_i)
    \right\|_{\Hcal} \\[1em]
\end{align*}
And so similarly to above, the second term becomes
\medmuskip=0mu
\thinmuskip=0mu
\thickmuskip=0mu
\begin{align*}
&\frac{1}{n} \sum_{i=1}^n
    \left\| P_{\hat{V}_d^m} \phi'(x_i) \right\|_{\Hcal}^2 = \\[1em]
    &\frac{1}{n} \sum_{i=1}^n
    \left(
    \sum_{j=1}^d
    \frac{1
    }{\sqrt{m \hat{\lambda}_j^m}}
    \sum_{k=1}^m
    u_{j,k}^m
    \left(
        k_{k,i}
        - \frac{1}{n}\sum_{\ell=1}^n k_{k,\ell}
        - \frac{1}{n}\sum_{r=1}^n \kappa_{m,i}^T K_{mm}^{-1} \kappa_{m,r}
        + \frac{1}{n^2}\sum_{\substack{\ell=1 \\ r=1}}^n \kappa_{m,\ell}^T K_{mm}^{-1} \kappa_{m,r}
    \right)
    \right)^2 \\[1em]
    &=\frac{1}{n \cdot m} \tr(K'_{nm}U_d^m \Lambda_d^{m \, \,-1} U_d^{m \,T} K'_{mn})
\end{align*}

with $U^m_d\Lambda^m_d U_d^{m \, T}$ the truncated eigendecomposition of $\frac{1}{m} K'_{mm}$.

\end{proof}

\begin{proof}[\textbf{Proof of Proposition \ref{prop:maj}}]

Since $\hat{V}_d^m \subset \Hcal_S$ for any $d$ and by Theorem~\ref{thm:nystrompca}
\begin{align*}
    \widetilde{\lambda}_{<d}
    &= \max_{\substack{\dim(V) = d \\ a + V \subset \Hcal_S}}
        \frac{1}{n} \sum_{i=1}^n\|P_{a + V} z_i\|^2_{\Hcal}
    = 
    \max_{\substack{\dim(V) = d \\ V \, \subset \Hcal_S \\ a \, \in \Hcal_S}} \frac{1}{n} \sum_{i=1}^n\|P_V (z_i - a)\|^2_{\Hcal} \\[1em]
    &\ge
    \frac{1}{n} \sum_{i=1}^n\|P_{\hat{V}_d^m} (z_i - \phi_0)\|^2_{\Hcal}
    = \hat{\lambda}_{<d}^{m,n}
\end{align*}
The case $d = m$ follows since both $\langle \{\hat{\phi}_j^{m,n}\}_{j=1}^m \rangle$ and $\langle \{\widetilde{\phi}_j\}_{j=1}^m \rangle$ capture the full variance of the data in $\Hcal_S$.

\end{proof}

\begin{proof}[\textbf{Proof of Proposition \ref{prop:trueeq}}]

By the previous proposition we have $\widetilde{V}_m = \hat{V}_m^m$ for a fixed $\omega$ and so we will have $\widetilde{V}_m \overset{d}{=} \hat{V}_m^m$ if $\{X_{i_1}, X_{i_2}, ..., X_{i_m}\} \overset{d}{=} \{X_1, X_2, ..., X_m\}$, where $S = \{i_1, i_2, ..., i_m\}$ are the indices for the subsampled data points. By the law of total probability
\begin{align*}
	&\p(\{X_{i_1} \le a_1, X_{i_2} \le a_2, \; ..., \; X_{i_m} \le a_m \}) \\[1em]
	&= \sum_S \p(\{X_{i_1} \le a_1, X_{i_2} \le a_2, \; ..., \; X_{i_m} \le a_m \} | S) \p(S) \\[1em]
	&= \sum_S \p(\{X_1 \le a_1, X_2 \le a_2, \; ..., \; X_m \le a_m \} | S) \p(S)
\end{align*}
since conditional on the sample $S$, we have $m$ random variables generated according to $\p_X$, which we can take to be $X_1, X_2, ..., X_m$. If the subsampling is independent of the data then
\begin{align*}
	&\sum_S \p(\{X_1 \le a_1, X_2 \le a_2, \; ..., \; X_m \le a_m \} | S) \p(S) \\[0.7em]
	&= \p(\{X_1 \le a_1, X_2 \le a_2, \; ..., \; X_m \le a_m \}) \sum_S \p(S)
	= \prod_{k=1}^m \p(\{X_k \le a_k\})
\end{align*}
so the subsampled data points are generated i.i.d. from $\p_X$. We can therefore conclude that $\widetilde{V}_m \overset{d}{=} \hat{V}_m^m$. Since $Z$ has the same distribution $\p_Z$ regardless of the subspace and since $\widetilde{V}_m \overset{d}{=} \hat{V}_m^m$ we have $P_{\widetilde{V}_m}Z' \overset{d}{=} P_{\hat{V}_m^m}Z'$ and can conclude that, since $Z$ is square-integrable
\[
    \E [\| P_{\widetilde{V}_m}Z' - Z'\|^2_{\Hcal}] = \E [ \| P_{\hat{V}_m^m}Z' - Z' \|^2_{\Hcal}]
\]
and so $R(\widetilde{V}_m) = R(\hat{V}^m_m)$ when $p(S \, | \, x_1, x_2, ..., x_n) = p(S)$.

\end{proof}

\begin{proof}[\textbf{Proof of Theorem \ref{thm:prob}}]

The difference in errors can be rewritten through
\begin{align*}
    R_n(\widetilde{V}_d) - R_n(\hat{V}_d)
    &= \min_{\substack{\dim(V)=d \\ V \subset \Hcal_S}}
\frac{1}{n}
\sum_{i=1}^n
\|P_Vz_i - z_i\|_{\Hcal}^2 -
    \min_{\dim(V)=d}
\frac{1}{n}
\sum_{i=1}^n
\|P_Vz_i - z_i\|_{\Hcal}^2 \\[1em]
&= \max_{\dim(V)=d}
\frac{1}{n}
\sum_{i=1}^n
\|P_Vz_i\|_{\Hcal}^2 -
    \max_{\substack{\dim(V)=d \\ V \subset \Hcal_S}}
\frac{1}{n}
\sum_{i=1}^n \|P_Vz_i\|_{\Hcal}^2 \\[1em]
% New line
    &=
\frac{1}{n}
\sum_{i=1}^n \|
    (P_{\Hcal_S} +
    P_{\Hcal_S^{\perp}})
    P_{\hat{V}_d}z_i
    \|_{\Hcal}^2 -
    \max_{\substack{\dim(V)=d \\ V \subset \Hcal_S}}
\frac{1}{n}
\sum_{i=1}^n
\|P_Vz_i\|_{\Hcal}^2 \\[1em]
    &\le
\frac{1}{n}
\sum_{i=1}^n \|
    P_{\Hcal_S}P_{\hat{V}_d}z_i
    \|_{\Hcal}^2 +
\frac{1}{n}
    \sum_{i=1}^n \|
    P_{\Hcal_S^{\perp}}P_{\hat{V}_d}z_i
    \|_{\Hcal}^2 -
    \max_{\substack{\dim(V)=d \\ V \subset \Hcal_S}}
\frac{1}{n}
\sum_{i=1}^n
\|P_Vz_i\|_{\Hcal}^2 \\[1em]
    &\le
\frac{1}{n}
    \sum_{i=1}^n \|
    P_{ \Hcal_S^{\perp}}P_{\hat{V}_d}z_i
    \|_{\Hcal}^2
\end{align*}
Expanding the projection operator $P_{\hat{V}_d}$ we obtain
\begin{align*}
    &\frac{1}{n}
    \sum_{i=1}^n \|
P_{\Hcal_S^{\perp}}
    P_{\hat{V}_d} z_i
    \|_{\Hcal_S}^2
=
\frac{1}{n}
    \sum_{i=1}^n \left\|
P_{\Hcal_S^{\perp}}
\sum_{j=1}^d \langle z_i, \hat{\phi}^n_j \rangle_{\Hcal} \hat{\phi}^n_j
    \right\|_{\Hcal}^2 \\[1em]
&= \frac{1}{n}
    \sum_{i=1}^n \left\|
\sum_{j=1}^d \langle z_i, \hat{\phi}^n_j \rangle_{\Hcal}
P_{\Hcal_S^{\perp}}
\hat{\phi}^n_j
    \right\|_{\Hcal}^2
\le
\frac{1}{n}
    \sum_{i=1}^n
\sum_{j=1}^d
\left\|
\langle z_i, \hat{\phi}^n_j \rangle_{\Hcal}
P_{\Hcal_S^{\perp}}
\hat{\phi}^n_j
    \right\|_{\Hcal}^2
\end{align*}
The last inequality is fairly sharp. It becomes an equality without the projection $P_{\Hcal_S^{\perp}}$, and the further the projection is from the identity, the smaller the norm of $P_{\Hcal_S^{\perp}}\hat{\phi}^n_j$. Now we have
\begin{align*}
    &\frac{1}{n}
    \sum_{i=1}^n
\sum_{j=1}^d
\left\|
    \langle z_i, \hat{\phi}^n_j \rangle_{\Hcal}
P_{\Hcal_S^{\perp}}
\hat{\phi}^n_j
    \right\|_{\Hcal}^2
=
\frac{1}{n}
    \sum_{i=1}^n
\sum_{j=1}^d
|\langle z_i, \hat{\phi}^n_j \rangle_{\Hcal}|^2
\left\|
P_{\Hcal_S^{\perp}}
\hat{\phi}^n_j
    \right\|_{\Hcal}^2 \\[1em]
    &=
\sum_{j=1}^d
\left(
\frac{1}{n}
    \sum_{i=1}^n
|\langle z_i, \hat{\phi}^n_j \rangle_{\Hcal}|^2
\right)
\left\|
P_{\Hcal_S^{\perp}}
\hat{\phi}^n_j
    \right\|_{\Hcal}^2
=
\sum_{j=1}^d
\hat{\lambda}_j^n
\left\|
P_{\Hcal_S^{\perp}}
\hat{\phi}^n_j
    \right\|_{\Hcal}^2
\end{align*}
Expanding the other projection operator we get
\begin{align*}
\sum_{j=1}^d
\hat{\lambda}_j^n
\left\|
P_{\Hcal_S^{\perp}}
\hat{\phi}^n_j
    \right\|_{\Hcal}^2
=
\sum_{j=1}^d
\hat{\lambda}_j^n
\left\|
\hat{\phi}^n_j
-
P_{\Hcal_S}
\hat{\phi}^n_j
    \right\|_{\Hcal}^2
=
\sum_{j=1}^d
\hat{\lambda}_j^n
\left\|
\hat{\phi}^n_j
-
\sum_{k=1}^m
\langle
\hat{\phi}^n_j,
\hat{\phi}_k^m
\rangle_{\Hcal}
\hat{\phi}_k^m
    \right\|_{\Hcal}^2 
\end{align*}
We have, for any $j$
\begin{align*}
\left\|
\hat{\phi}^n_j
-
\sum_{k=1}^m
\langle
\hat{\phi}^n_j,
\hat{\phi}_k^m
\rangle_{\Hcal}
\hat{\phi}_k^m
    \right\|_{\Hcal}^2 
\le
\left\|
\hat{\phi}^n_j
-
\langle
\hat{\phi}^n_j,
\hat{\phi}_j^m
\rangle_{\Hcal}
\hat{\phi}_j^m
    \right\|_{\Hcal}^2
= 1 - 
\langle
\hat{\phi}^n_j,
\hat{\phi}_j^m
\rangle_{\Hcal}^2
=
\sin^2\theta_j
\end{align*}
Then by the Davis-Kahan theorem \citep[Corollary 1]{yu2015useful} (also see \cite{davis1970rotation}), defining $\hat{\lambda}^m_0 := +\infty$ and $\hat{\lambda}^m_{m+1} := -\infty$
\begin{align*}
  \sin \theta_j
  \le \frac{2\| C_n - C_m\|_{\hs}}{\min\{
    \hat{\lambda}_{j-1}^m - \hat{\lambda}_j^m,
    \hat{\lambda}_j^m - \hat{\lambda}_{j+1}^m
  \}}
  \wedge 1 =: \sqrt{D_j}
\end{align*}
%%%%%%%%%%%%%%%%%%%%%%%%%%%%%%%%%%%%%%%%%%%%%%%%%%%%%%%%%%%%%%%%%%%%%%%
%%% Keeping all dimensions (but doesn't make much difference... %%%%%%%
%%%%%%%%%%%%%%%%%%%%%%%%%%%%%%%%%%%%%%%%%%%%%%%%%%%%%%%%%%%%%%%%%%%%%%%
\begin{comment}
But we can also write
% NOTE we actually don't need absolute value, the RHS will always be positive
\begin{align*}
\left\|
\hat{\phi}^n_j
-
\sum_{k=1}^m
\langle
\hat{\phi}^n_j,
\hat{\phi}_k^m
\rangle_{\Hcal}
\hat{\phi}_k^m
    \right\|_{\Hcal}^2 
=
1 -
\sum_{k=1}^m 
\langle
\hat{\phi}^n_j,
\hat{\phi}_k^m
\rangle_{\Hcal}^2
=
1 - \sum_{k=1}^m\left( 1 - \sin^2\theta_{j,k} \right)
\le
\sum_{k=1}^m\sin^2\theta_{j,k} - (m-1)
\end{align*}
Again by the Davis-Kahan theorem \emph{TODO}
\[
\sin\theta_{j,k}
\le
\frac{2\|C_n - C_m\|_{\hs}}{\min \left\{\hat{\lambda}_1^m - \hat{\lambda}_2^m, \hat{\lambda}_2^m - \hat{\lambda}_3^m \right\}} =: c_2
\]
And so
\begin{align*}
&
\sum_{k=1}^m\sin^2\theta_{j,k} - (m-1)
\\[1em]
&\le
...
\\[1em]
=: D_j
\end{align*}
\end{comment}
% TODO: we can use the sup norm, we need this to apply the D-K variant, and then we get the square root of B, sharpening (unless my operator bound sharpening doesn't work, although should be immediate...)
Then by Lidskii's inequality \citep[Chapter~3, Theorem~6.11]{kato2013perturbation}
\begin{align*}
    \sum_{j=1}^d
     \hat{\lambda}_j^n \cdot D_j
=
    \sum_{j=1}^d
    \hat{\lambda}_j^m \cdot D_j
    +
    \sum_{j=1}^d
    \left(
    \hat{\lambda}^n_j - \hat{\lambda}_j^m
    \right) \cdot D_j
\le
    \sum_{j=1}^d
    \hat{\lambda}_j^m \cdot D_j
    +
    \|C_n - C_m\|_{\hs}
    \max_{1\le k \le d}D_k
\end{align*}
Now the only unknown and random quantity is $\| C_n - C_m \|_{\hs}$. It depends both on the unobserved data points $z_{m+1}, z_{m+2}, ..., z_n$ and the observed ones $z_1, z_2, ..., z_m$. We split these up into two terms
\begin{align*}
    \|C_n - C_m  \|_{\hs}
        &= \left\| \frac{1}{n}\sum_{i=1}^n z_i \otimes z_i
        - \frac{1}{m}\sum_{r=1}^m z_r \otimes z_r
        \right\|_{\Htsr} \\[0.8em]
       &= \left\| \frac{1}{n}\sum_{i=m+1}^n z_i \otimes z_i
        - \frac{n  - m}{nm}\sum_{r=1}^m z_r \otimes z_r
        \right\|_{\Htsr} \\[0.8em]
        &= \frac{n-m}{n} \left\|
            \frac{1}{n-m}\sum_{i=m+1}^n z_i \otimes z_i
            - \frac{1}{m}\sum_{r=1}^m z_r \otimes z_r
        \right\|_{\Htsr} \\[0.8em]
        &=  \frac{n-m}{n} \left\|
        C_{n-m} - C_m
         \right\|_{\hs}
\end{align*}
If we let $Y_i = z_i \otimes z_i - C_m$, then $\frac{1}{n-m}\sum_{i=m+1}^n Y_i = C_{n-m} - C_m$ and the random variables $Y_i$ have zero expectation with respect to $C_{n-m}$ across hypothesized repeated realizations of $C_m$, and they are bounded by $\sqrt{2}B:= \sqrt{2}\sup_xk(x,x)$ by Lemma~\ref{lem:opbound} since both $z_i \otimes z_i$ and $C_m$ are positive. Then by Hoeffding's inequality in Banach spaces \citep[Theorem~3.5]{pinelis1994optimum}, we have that with probability at least $1-2e^{-\delta}$, over infinite repetitions of the experiment yielding $C_m$
\[
    \frac{n-m}{n}
        \left\|
            \frac{1}{n-m}\sum_{i=m+1}^n z_i \otimes z_i
	    -
            \frac{1}{m}\sum_{r=1}^m z_r \otimes z_r
         \right\|_{\Htsr}
    \le \frac{n-m}{n}\frac{ 2B\sqrt{\delta}}{\sqrt{n-m}}
    %=: D
\]
% TODO maybe remove this to make the proof shorter, since it's not necessary
\begin{comment}
and so also with probability at least $1-2e^{-\delta}$ that
\[
D_j
  \le \frac{(2D)^2}{\min\{
    \hat{\lambda}_{j-1}^m - \hat{\lambda}_j^m,
    \hat{\lambda}_j^m - \hat{\lambda}_{j+1}^m
  \}^2}
  \wedge 1
\]
\end{comment}
We recall that the eigenvalues of the empirical covariance operator equal the eigenvalues of the kernel matrix $\frac{1}{m} K_{mm}$, which completes the proof.

\end{proof}

\newpage

\fancyhf{}
\fancyfoot[C]{{\Small \thepage}}
\fancyfoot[LR]{}
\interlinepenalty=10000

\vspace{-30pt}

% Spacing after heading
\renewcommand{\bibpreamble}{\vskip0.9cm}

% Spacing between items
\setlength{\bibsep}{12pt plus 1ex}

\bibliography{mybibfile}

\begin{thebibliography}{88}
\providecommand{\natexlab}[1]{#1}
\providecommand{\url}[1]{\texttt{#1}}
\expandafter\ifx\csname urlstyle\endcsname\relax
  \providecommand{\doi}[1]{doi: #1}\else
  \providecommand{\doi}{doi: \begingroup \urlstyle{rm}\Url}\fi

\bibitem[Aharon et~al.(2006)Aharon, Elad, and Bruckstein]{aharon2006k}
M.~Aharon, M.~Elad, and A.~Bruckstein.
\newblock K-{SVD}: An algorithm for designing overcomplete dictionaries for
  sparse representation.
\newblock \emph{IEEE Transactions on Signal Processing}, 54\penalty0
  (11):\penalty0 4311--4322, 2006.
\newblock
  \newline{\Small\url{https://ieeexplore.ieee.org/stamp/stamp.jsp?arnumber=1710377}}.

\bibitem[Amini and Wainwright(2012)]{amini2012sampled}
A.~A. Amini and M.~J. Wainwright.
\newblock Sampled forms of functional {PCA} in reproducing kernel {Hilbert}
  spaces.
\newblock \emph{The Annals of Statistics}, 40\penalty0 (5):\penalty0
  2483--2510, 2012.
\newblock
  \newline{\Small\url{https://projecteuclid.org/journals/annals-of-statistics/volume-40/issue-5/Sampled-forms-of-functional-PCA-in-reproducing-kernel-Hilbert-spaces/10.1214/12-AOS1033.pdf}}.

\bibitem[Bach and Jordan(2002)]{bach2002kernel}
F.~R. Bach and M.~I. Jordan.
\newblock Kernel independent component analysis.
\newblock \emph{The Journal of Machine Learning Research}, 3\penalty0
  (Jul):\penalty0 1--48, 2002.
\newblock
  \newline{\Small\url{http://www.jmlr.org/papers/volume3/bach02a/bach02a.pdf}}.

\bibitem[Balcan et~al.(2016)Balcan, Liang, Song, Woodruff, and
  Xie]{balcan2016communication}
M.~F. Balcan, Y.~Liang, L.~Song, D.~Woodruff, and B.~Xie.
\newblock Communication efficient distributed kernel principal component
  analysis.
\newblock In \emph{Proceedings of the 22nd ACM SIGKDD International Conference
  on Knowledge Discovery and Data Mining}, pages 725--734, 2016.
\newblock
  \newline{\Small\url{https://dl.acm.org/doi/pdf/10.1145/2939672.2939796}}.

\bibitem[Banach(1932)]{banach1932theorie}
S.~Banach.
\newblock Théorie des opérations linéaires.
\newblock \emph{Monografie Matematyczne}, 1932.
\newblock
  \newline{\Small\url{http://kielich.amu.edu.pl/Stefan_Banach/pdf/teoria-operacji-fr/banach-teorie-des-operations-lineaires.pdf}}.

\bibitem[Bartlett and Mendelson(2002)]{bartlett2002rademacher}
P.~L. Bartlett and S.~Mendelson.
\newblock Rademacher and {G}aussian complexities: risk bounds and structural
  results.
\newblock \emph{The Journal of Machine Learning Research}, 3\penalty0
  (Nov):\penalty0 463--482, 2002.
\newblock
  \newline{\Small\url{https://www.jmlr.org/papers/volume3/bartlett02a/bartlett02a.pdf}}.

\bibitem[Beaulieu(2020)]{beaulieu2020learning}
A.~Beaulieu.
\newblock \emph{Learning SQL: generate, manipulate, and retrieve data}.
\newblock O'Reilly Media, 3rd edition, 2020.
\newblock
  \newline{\Small\url{https://www.oreilly.com/library/view/learning-sql-3rd/9781492057604/}}.

\bibitem[Bengio et~al.(2004)Bengio, Delalleau, Roux, Paiement, Vincent, and
  Ouimet]{bengio2004learning}
Y.~Bengio, O.~Delalleau, N.~L. Roux, J.-F. Paiement, P.~Vincent, and M.~Ouimet.
\newblock Learning eigenfunctions links spectral embedding and kernel {PCA}.
\newblock \emph{Neural Computation}, 16\penalty0 (10):\penalty0 2197--2219,
  2004.
\newblock
  \newline{\Small\url{https://www.mitpressjournals.org/doi/pdfplus/10.1162/0899766041732396}}.

\bibitem[Besse(1991)]{besse1991approximation}
P.~Besse.
\newblock Approximation spline de l'analyse en composantes principales d'une
  variable al{\'e}atoire hilbertienne.
\newblock In \emph{Annales de la Facult{\'e} des sciences de Toulouse:
  Math{\'e}matiques}, volume~12, pages 329--349. Universit{\'e} Paul Sabatier,
  1991.
\newblock
  \newline{\Small\url{https://afst.centre-mersenne.org/article/AFST_1991_5_12_3_329_0.pdf}}.

\bibitem[Besse and Ramsay(1986)]{besse1986principal}
P.~Besse and J.~O. Ramsay.
\newblock Principal components analysis of sampled functions.
\newblock \emph{Psychometrika}, 51\penalty0 (2):\penalty0 285--311, 1986.
\newblock
  \newline{\Small\url{https://link.springer.com/content/pdf/10.1007/BF02293986.pdf}}.

\bibitem[Bhatia(1997)]{bhatia1997matrix}
R.~Bhatia.
\newblock \emph{Matrix analysis}, volume 169 of \emph{Graduate texts in
  mathematics}.
\newblock Springer Science \& Business Media, 1st edition, 1997.
\newblock
  \newline{\Small\url{https://link.springer.com/book/10.1007/978-1-4612-0653-8}}.

\bibitem[Blanchard and Zadorozhnyi(2019)]{blanchard2019concentration}
G.~Blanchard and O.~Zadorozhnyi.
\newblock Concentration of weakly dependent {B}anach-valued sums and
  applications to statistical learning methods.
\newblock \emph{Bernoulli}, 25\penalty0 (4B):\penalty0 3421--3458, 2019.
\newblock \newline{\Small\url{https://arxiv.org/pdf/1712.01934.pdf}}.

\bibitem[Blanchard et~al.(2007)Blanchard, Bousquet, and
  Zwald]{blanchard2007statistical}
G.~Blanchard, O.~Bousquet, and L.~Zwald.
\newblock Statistical properties of kernel principal component analysis.
\newblock \emph{Machine Learning}, 66\penalty0 (2-3):\penalty0 259--294, 2007.
\newblock
  \newline{\Small\url{https://link.springer.com/content/pdf/10.1007/s10994-006-6895-9.pdf}}.

\bibitem[Bollob{\'a}s(1999)]{bollobas1999linear}
B.~Bollob{\'a}s.
\newblock \emph{Linear analysis}.
\newblock Cambridge mathematical textbooks. Cambridge University Press, 2nd
  edition, 1999.
\newblock
  \newline{\Small\url{https://www.cambridge.org/core/books/linear-analysis/E43EE4282F2D8636117A47A4F110E8FE}}.

\bibitem[Bouvrie and Hamzi(2012)]{bouvrie2012kernel}
J.~Bouvrie and B.~Hamzi.
\newblock Kernel methods for the approximation of some key quantities of
  nonlinear systems.
\newblock In \emph{Proc. American Control Conference}, 2012.
\newblock \newline{\Small\url{https://arxiv.org/pdf/1204.0563.pdf}}.

\bibitem[Carratino et~al.(2021)Carratino, Vigogna, Calandriello, and
  Rosasco]{carratino2021park}
L.~Carratino, S.~Vigogna, D.~Calandriello, and L.~Rosasco.
\newblock Par{K}: Sound and efficient kernel ridge regression by feature space
  partitions.
\newblock \emph{Advances in Neural Information Processing Systems},
  34:\penalty0 6430--6441, 2021.
\newblock
  \newline{\Small\url{https://proceedings.neurips.cc/paper/2021/file/32b9e74c8f60958158eba8d1fa372971-Paper.pdf}}.

\bibitem[Cohn(1980)]{cohn1980measure}
D.~L. Cohn.
\newblock \emph{Measure theory}, volume 165 of \emph{Birkhäuser Advanced Texts
  Basler Lehrbucher}.
\newblock Springer Science \& Business Media, 2nd edition, 1980.
\newblock
  \newline{\Small\url{https://link.springer.com/book/10.1007/978-1-4614-6956-8}}.

\bibitem[Dauxois et~al.(1982)Dauxois, Pousse, and
  Romain]{dauxois1982asymptotic}
J.~Dauxois, A.~Pousse, and Y.~Romain.
\newblock Asymptotic theory for the principal component analysis of a vector
  random function: some applications to statistical inference.
\newblock \emph{Journal of Multivariate Analysis}, 12\penalty0 (1):\penalty0
  136--154, 1982.
\newblock \newline{\Small\url{https://core.ac.uk/download/pdf/82501258.pdf}}.

\bibitem[Davies(2007)]{davies2007linear}
E.~B. Davies.
\newblock \emph{Linear operators and their spectra}, volume 106 of
  \emph{Cambridge studies in advanced mathematics}.
\newblock Cambridge University Press, 1st edition, 2007.
\newblock
  \newline{\Small\url{https://www.cambridge.org/core/books/linear-operators-and-their-spectra/6DDA33D1D7032F9EBB41194F33C18A69}}.

\bibitem[Davis and Kahan(1970)]{davis1970rotation}
C.~Davis and W.~M. Kahan.
\newblock The rotation of eigenvectors by a perturbation. {III}.
\newblock \emph{SIAM Journal on Numerical Analysis}, 7\penalty0 (1):\penalty0
  1--46, 1970.
\newblock \newline{\Small\url{https://epubs.siam.org/doi/pdf/10.1137/0707001}}.

\bibitem[Davison and Hinkley(1997)]{davison1997bootstrap}
A.~C. Davison and D.~V. Hinkley.
\newblock \emph{Bootstrap methods and their application}.
\newblock Cambridge series in statistical and probabilistic mathematics.
  Cambridge University Press, 1st edition, 1997.
\newblock
  \newline{\Small\url{https://www.cambridge.org/core/books/bootstrap-methods-and-their-application/ED2FD043579F27952363566DC09CBD6A}}.

\bibitem[De~Silva and Tenenbaum(2004)]{de2004sparse}
V.~De~Silva and J.~B. Tenenbaum.
\newblock Sparse multidimensional scaling using landmark points.
\newblock Technical report, Stanford University, 2004.
\newblock
  \newline{\Small\url{http://graphics.stanford.edu/courses/cs468-05-winter/Papers/Landmarks/Silva_landmarks5.pdf}}.

\bibitem[De~Vito et~al.(2005{\natexlab{a}})De~Vito, Caponnetto, and
  Rosasco]{de2005model}
E.~De~Vito, A.~Caponnetto, and L.~Rosasco.
\newblock Model selection for regularized least-squares algorithm in learning
  theory.
\newblock \emph{Foundations of Computational Mathematics}, 5\penalty0
  (1):\penalty0 59--85, 2005{\natexlab{a}}.
\newblock
  \newline{\Small\url{https://web.mit.edu/lrosasco/www/publications/model_focm.pdf}}.

\bibitem[De~Vito et~al.(2005{\natexlab{b}})De~Vito, Rosasco, Caponnetto,
  De~Giovannini, and Odone]{de2005learning}
E.~De~Vito, L.~Rosasco, A.~Caponnetto, U.~De~Giovannini, and F.~Odone.
\newblock Learning from examples as an inverse problem.
\newblock \emph{The Journal of Machine Learning Research}, 6\penalty0 (5),
  2005{\natexlab{b}}.
\newblock
  \newline{\Small\url{https://www.jmlr.org/papers/volume6/devito05a/devito05a.pdf}}.

\bibitem[Dua and Graff(2017)]{dua2019uci}
D.~Dua and C.~Graff.
\newblock {UCI} machine learning repository, 2017.
\newblock \newline{\Small\url{http://archive.ics.uci.edu/ml}}.

\bibitem[Frangella et~al.(2021)Frangella, Tropp, and
  Udell]{frangella2021randomized}
Z.~Frangella, J.~A. Tropp, and M.~Udell.
\newblock Randomized {N}ystr{\"o}m preconditioning.
\newblock \emph{arXiv preprint arXiv:2110.02820}, 2021.
\newblock \newline{\Small\url{https://arxiv.org/pdf/2110.02820.pdf}}.

\bibitem[Fuller(1980)]{fuller1980properties}
W.~A. Fuller.
\newblock Properties of some estimators for the errors-in-variables model.
\newblock \emph{The Annals of Statistics}, pages 407--422, 1980.
\newblock
  \newline{\Small\url{https://projecteuclid.org/download/pdf_1/euclid.aos/1176344961}}.

\bibitem[Garreau et~al.(2017)Garreau, Jitkrittum, and
  Kanagawa]{garreau2017large}
D.~Garreau, W.~Jitkrittum, and M.~Kanagawa.
\newblock Large sample analysis of the median heuristic.
\newblock \emph{arXiv preprint arXiv:1707.07269}, 2017.
\newblock \newline{\Small\url{https://arxiv.org/pdf/1707.07269.pdf}}.

\bibitem[Giraldo et~al.(2014)Giraldo, Rao, and Principe]{giraldo2014measures}
L.~G.~S. Giraldo, M.~Rao, and J.~C. Principe.
\newblock Measures of entropy from data using infinitely divisible kernels.
\newblock \emph{IEEE Transactions on Information Theory}, 61\penalty0
  (1):\penalty0 535--548, 2014.
\newblock \newline{\Small\url{https://ieeexplore.ieee.org/document/6954500}}.

\bibitem[Gisbrecht and Schleif(2015)]{gisbrecht2015metric}
A.~Gisbrecht and F.-M. Schleif.
\newblock Metric and non-metric proximity transformations at linear costs.
\newblock \emph{Neurocomputing}, 167:\penalty0 643--657, 2015.
\newblock \newline{\Small\url{https://arxiv.org/pdf/1411.1646}}.

\bibitem[Gittens and Mahoney(2016)]{gittens2016revisiting}
A.~Gittens and M.~W. Mahoney.
\newblock Revisiting the {N}ystr{\"o}m method for improved large-scale machine
  learning.
\newblock \emph{The Journal of Machine Learning Research}, 17\penalty0
  (1):\penalty0 3977--4041, 2016.
\newblock
  \newline{\Small\url{http://www.jmlr.org/papers/volume17/gittens16a/gittens16a.pdf}}.

\bibitem[Golts and Elad(2016)]{golts2016linearized}
A.~Golts and M.~Elad.
\newblock Linearized kernel dictionary learning.
\newblock \emph{IEEE Journal of Selected Topics in Signal Processing},
  10\penalty0 (4):\penalty0 726--739, 2016.
\newblock \newline{\Small\url{https://arxiv.org/pdf/1509.05634.pdf}}.

\bibitem[Golub and Van~Loan(2013)]{golub2013matrix}
G.~H. Golub and C.~F. Van~Loan.
\newblock \emph{Matrix computations}.
\newblock Johns Hopkins University Press, 4th edition, 2013.
\newblock
  \newline{\Small\url{https://jhupbooks.press.jhu.edu/title/matrix-computations}}.

\bibitem[Graham and Talay(2011)]{graham2011simulation}
C.~Graham and D.~Talay.
\newblock \emph{Simulation stochastique et m{\'e}thodes de Monte-Carlo}.
\newblock {\'E}cole Polytechnique, D{\'e}partement de Math{\'e}matiques
  Appliqu{\'e}es, 2011.
\newblock \newline{\Small\url{https://hal.archives-ouvertes.fr/hal-00602795}}.

\bibitem[Haddouche et~al.(2020)Haddouche, Guedj, Rivasplata, and
  Shawe-Taylor]{haddouche2020upper}
M.~Haddouche, B.~Guedj, O.~Rivasplata, and J.~Shawe-Taylor.
\newblock Upper and lower bounds on the performance of kernel {PCA}.
\newblock \emph{arXiv preprint arXiv:2012.10369}, 2020.
\newblock \newline{\Small\url{https://arxiv.org/pdf/2012.10369.pdf}}.

\bibitem[Hall and Hosseini-Nasab(2006)]{hall2006properties}
P.~Hall and M.~Hosseini-Nasab.
\newblock On properties of functional principal components analysis.
\newblock \emph{Journal of the Royal Statistical Society: Series B (Statistical
  Methodology)}, 68\penalty0 (1):\penalty0 109--126, 2006.
\newblock
  \newline{\Small\url{https://rss.onlinelibrary.wiley.com/doi/pdf/10.1111/j.1467-9868.2005.00535.x}}.

\bibitem[Hall et~al.(2006)Hall, M{\"u}ller, and Wang]{hall2006properties2}
P.~Hall, H.-G. M{\"u}ller, and J.-L. Wang.
\newblock Properties of principal component methods for functional and
  longitudinal data analysis.
\newblock \emph{The Annals of Statistics}, pages 1493--1517, 2006.
\newblock
  \newline{\Small\url{https://projecteuclid.org/download/pdfview_1/euclid.aos/1152540756}}.

\bibitem[Hofmann et~al.(2008)Hofmann, Sch{\"o}lkopf, and
  Smola]{hofmann2008kernel}
T.~Hofmann, B.~Sch{\"o}lkopf, and A.~J. Smola.
\newblock Kernel methods in machine learning.
\newblock \emph{The Annals of Statistics}, pages 1171--1220, 2008.
\newblock
  \newline{\Small\url{https://projecteuclid.org/journals/annals-of-statistics/volume-36/issue-3/Kernel-methods-in-machine-learning/10.1214/009053607000000677.pdf}}.

\bibitem[Hout et~al.(2013)Hout, Papesh, and
  Goldinger]{hout2013multidimensional}
M.~C. Hout, M.~H. Papesh, and S.~D. Goldinger.
\newblock Multidimensional scaling.
\newblock \emph{Wiley Interdisciplinary Reviews: Cognitive Science}, 4\penalty0
  (1):\penalty0 93--103, 2013.
\newblock
  \newline{\Small\url{https://onlinelibrary.wiley.com/doi/pdf/10.1002/wcs.1203}}.

\bibitem[Hyv{\"a}rinen and Oja(2000)]{hyvarinen2000independent}
A.~Hyv{\"a}rinen and E.~Oja.
\newblock Independent component analysis: algorithms and applications.
\newblock \emph{Neural Networks}, 13\penalty0 (4-5):\penalty0 411--430, 2000.
\newblock
  \newline{\Small\url{https://www.sciencedirect.com/science/article/pii/S0893608000000265}}.

\bibitem[Iosifidis and Gabbouj(2016)]{iosifidis2016nystrom}
A.~Iosifidis and M.~Gabbouj.
\newblock Nystr{\"o}m-based approximate kernel subspace learning.
\newblock \emph{Pattern Recognition}, 57:\penalty0 190--197, 2016.
\newblock
  \newline{\Small\url{https://www.sciencedirect.com/science/article/pii/S0031320316300036}}.

\bibitem[Jolliffe(2002)]{jolliffe2002principal}
I.~T. Jolliffe.
\newblock \emph{Principal component analysis}.
\newblock Springer Science \& Business Media, 2nd edition, 2002.
\newblock
  \newline{\Small\url{http://cda.psych.uiuc.edu/statistical_learning_course/Jolliffe%20I.%20Principal%20Component%20Analysis%20(2ed.,%20Springer,%202002)(518s)_MVsa_.pdf}}.

\bibitem[Jolliffe and Cadima(2016)]{jolliffe2016principal}
I.~T. Jolliffe and J.~Cadima.
\newblock Principal component analysis: a review and recent developments.
\newblock \emph{Philosophical Transactions of the Royal Society A:
  Mathematical, Physical and Engineering Sciences}, 374\penalty0
  (2065):\penalty0 20150202, 2016.
\newblock
  \newline{\Small\url{https://royalsocietypublishing.org/doi/pdf/10.1098/rsta.2015.0202}}.

\bibitem[Kato(2013)]{kato2013perturbation}
T.~Kato.
\newblock \emph{Perturbation theory for linear operators}, volume 132 of
  \emph{Classics in mathematics}.
\newblock Springer Science \& Business Media, 2013.
\newblock
  \newline{\Small\url{https://link.springer.com/book/10.1007/978-3-642-66282-9}}.

\bibitem[Koltchinskii and Gin{\'e}(2000)]{koltchinskii2000random}
V.~Koltchinskii and E.~Gin{\'e}.
\newblock Random matrix approximation of spectra of integral operators.
\newblock \emph{Bernoulli}, 6\penalty0 (1):\penalty0 113--167, 2000.
\newblock
  \newline{\Small\url{https://projecteuclid.org/download/pdf_1/euclid.bj/1082665383}}.

\bibitem[Kreyszig(1989)]{kreyszig1989introductory}
E.~Kreyszig.
\newblock \emph{Introductory functional analysis with applications}.
\newblock Wiley, 1st edition, 1989.
\newblock
  \newline{\Small\url{https://www.wiley.com/en-gb/Introductory+Functional+Analysis+with+Applications-p-9780471504597}}.

\bibitem[Ledoux and Talagrand(2013)]{ledoux2013probability}
M.~Ledoux and M.~Talagrand.
\newblock \emph{Probability in Banach spaces: isoperimetry and processes}.
\newblock Springer Science \& Business Media, 2013.
\newblock \newline{\Small\url{https://www.springer.com/gp/book/9783642202117}}.

\bibitem[Liang and Rakhlin(2020)]{liang2020just}
T.~Liang and A.~Rakhlin.
\newblock Just interpolate: kernel “ridgeless” regression can generalize.
\newblock \emph{The Annals of Statistics}, 48\penalty0 (3):\penalty0
  1329--1347, 2020.
\newblock \newline{\Small\url{https://arxiv.org/pdf/1808.00387.pdf}}.

\bibitem[Lodhi et~al.(2002)Lodhi, Saunders, Shawe-Taylor, Cristianini, and
  Watkins]{lodhi2002text}
H.~Lodhi, C.~Saunders, J.~Shawe-Taylor, N.~Cristianini, and C.~Watkins.
\newblock Text classification using string kernels.
\newblock \emph{The Journal of Machine Learning Research}, 2\penalty0
  (Feb):\penalty0 419--444, 2002.
\newblock
  \newline{\Small\url{http://www.jmlr.org/papers/volume2/lodhi02a/lodhi02a.pdf}}.

\bibitem[Ma and Belkin(2017)]{ma2017diving}
S.~Ma and M.~Belkin.
\newblock Diving into the shallows: a computational perspective on large-scale
  shallow learning.
\newblock In \emph{Advances in Neural Information Processing Systems}, pages
  3778--3787, 2017.
\newblock
  \newline{\Small\url{https://papers.nips.cc/paper/6968-diving-into-the-shallows-a-computational-perspective-on-large-scale-shallow-learning.pdf}}.

\bibitem[McDiarmid(1989)]{mcdiarmid1989method}
C.~McDiarmid.
\newblock On the method of bounded differences.
\newblock \emph{Surveys in Combinatorics}, 141\penalty0 (1):\penalty0 148--188,
  1989.
\newblock
  \newline{\Small\url{https://www.cambridge.org/no/academic/subjects/mathematics/discrete-mathematics-information-theory-and-coding/surveys-combinatorics-1989-invited-papers-twelfth-british-combinatorial-conference?format=PB&isbn=9780521378239}}.

\bibitem[Meanti et~al.(2020)Meanti, Carratino, Rosasco, and
  Rudi]{meanti2020kernel}
G.~Meanti, L.~Carratino, L.~Rosasco, and A.~Rudi.
\newblock Kernel methods through the roof: handling billions of points
  efficiently.
\newblock \emph{arXiv preprint arXiv:2006.10350}, 2020.
\newblock \newline{\Small\url{https://arxiv.org/abs/2006.10350}}.

\bibitem[Mika et~al.(1999)Mika, Rätsch, Weston, Schölkopf, and
  Müller]{mika1999fisher}
S.~Mika, G.~Rätsch, J.~Weston, B.~Schölkopf, and K.-R. Müller.
\newblock Fisher discriminant analysis with kernels.
\newblock In \emph{Neural Networks for Signal Processing IX: Proceedings of the
  1999 IEEE Signal Processing Society Workshop (cat. no. 98th8468)}, pages
  41--48. IEEE, 1999.
\newblock
  \newline{\Small\url{https://ieeexplore.ieee.org/abstract/document/788121}}.

\bibitem[Neyman(1937)]{neyman1937outline}
J.~Neyman.
\newblock Outline of a theory of statistical estimation based on the classical
  theory of probability.
\newblock \emph{Philosophical Transactions of the Royal Society of London.
  Series A, Mathematical and Physical Sciences}, 236\penalty0 (767):\penalty0
  333--380, 1937.
\newblock
  \newline{\Small\url{https://royalsocietypublishing.org/doi/pdf/10.1098/rsta.1937.0005?download=true}}.

\bibitem[Neyman and Pearson(1933)]{neyman1933ix}
J.~Neyman and E.~S. Pearson.
\newblock On the problem of the most efficient tests of statistical hypotheses.
\newblock \emph{Philosophical Transactions of the Royal Society of London.
  Series A, Containing Papers of a Mathematical or Physical Character},
  231\penalty0 (694-706):\penalty0 289--337, 1933.
\newblock
  \newline{\Small\url{https://royalsocietypublishing.org/doi/pdf/10.1098/rsta.1933.0009?download=true}}.

\bibitem[Nystr{\"o}m(1930)]{nystrom1930praktische}
E.~J. Nystr{\"o}m.
\newblock {\"U}ber die praktische aufl{\"o}sung von integralgleichungen mit
  anwendungen auf randwertaufgaben.
\newblock \emph{Acta Mathematica}, 54\penalty0 (1):\penalty0 185--204, 1930.
\newblock
  \newline{\Small\url{https://link.springer.com/content/pdf/10.1007/BF02547521.pdf}}.

\bibitem[Paulsen and Raghupathi(2016)]{paulsen2016introduction}
V.~I. Paulsen and M.~Raghupathi.
\newblock \emph{An introduction to the theory of reproducing kernel Hilbert
  spaces}, volume 152 of \emph{Cambridge studies in advanced mathematics}.
\newblock Cambridge University Press, 2016.
\newblock
  \newline{\Small\url{https://www.cambridge.org/core/books/an-introduction-to-the-theory-of-reproducing-kernel-hilbert-spaces/C3FD9DF5F5C21693DD4ED812B531269A}}.

\bibitem[Pearson(1901)]{pearson1901principal}
K.~Pearson.
\newblock On lines and planes of closest fit to systems of points in space.
\newblock \emph{The London, Edinburgh, and Dublin Philosophical Magazine and
  Journal of Science}, 6\penalty0 (2):\penalty0 559--572, 1901.
\newblock
  \newline{\Small\url{https://www.tandfonline.com/doi/abs/10.1080/14786440109462720}}.

\bibitem[Pinelis(1994)]{pinelis1994optimum}
I.~Pinelis.
\newblock Optimum bounds for the distributions of martingales in {B}anach
  spaces.
\newblock \emph{The Annals of Probability}, pages 1679--1706, 1994.
\newblock
  \newline{\Small\url{https://projecteuclid.org/download/pdf_1/euclid.aop/1176988477}}.

\bibitem[Platt(2005)]{platt2005fastmap}
J.~Platt.
\newblock {FastMap}, {MetricMap}, and {Landmark} {MDS} are all {N}ystr{\"o}m
  algorithms.
\newblock In \emph{International Conference on Artificial Intelligence and
  Statistics}, pages 261--268. PMLR, 2005.
\newblock
  \newline{\Small\url{http://proceedings.mlr.press/r5/platt05a/platt05a.pdf}}.

\bibitem[Ramachandran and Tsokos(2015)]{ramachandran2020mathematical}
K.~M. Ramachandran and C.~P. Tsokos.
\newblock \emph{Mathematical statistics with applications in R}.
\newblock Academic Press, 2nd edition, 2015.
\newblock
  \newline{\Small\url{https://www.elsevier.com/books/mathematical-statistics-with-applications-in-r/ramachandran/978-0-12-417113-8}}.

\bibitem[Robert and Casella(2004)]{robert2013monte}
C.~Robert and G.~Casella.
\newblock \emph{Monte Carlo statistical methods}.
\newblock Springer Science \& Business Media, 2nd edition, 2004.
\newblock \newline{\Small\url{https://www.springer.com/gp/book/9780387212395}}.

\bibitem[Rosasco et~al.(2010)Rosasco, Belkin, and Vito]{rosasco2010learning}
L.~Rosasco, M.~Belkin, and E.~D. Vito.
\newblock On learning with integral operators.
\newblock \emph{The Journal of Machine Learning Research}, 11\penalty0
  (Feb):\penalty0 905--934, 2010.
\newblock
  \newline{\Small\url{http://www.jmlr.org/papers/volume11/rosasco10a/rosasco10a.pdf}}.

\bibitem[Rosipal and Trejo(2001)]{rosipal2001kernelb}
R.~Rosipal and L.~J. Trejo.
\newblock Kernel partial least squares regression in reproducing kernel
  {H}ilbert space.
\newblock \emph{The Journal of Machine Learning Research}, 2\penalty0
  (Dec):\penalty0 97--123, 2001.
\newblock
  \newline{\Small\url{https://www.jmlr.org/papers/volume2/rosipal01a/rosipal01a.pdf}}.

\bibitem[Rosipal et~al.(2000)Rosipal, Trejo, and Cichocki]{rosipal2000kernel}
R.~Rosipal, L.~J. Trejo, and A.~Cichocki.
\newblock \emph{Kernel principal component regression with {EM} approach to
  nonlinear principal components extraction}.
\newblock University of Paisley, 2000.
\newblock
  \newline{\Small\url{http://aiolos.um.savba.sk/~roman/Papers/tr00_2.pdf}}.

\bibitem[Rosipal et~al.(2001)Rosipal, Girolami, Trejo, and
  Cichocki]{rosipal2001kernel}
R.~Rosipal, M.~Girolami, L.~J. Trejo, and A.~Cichocki.
\newblock Kernel {PCA} for feature extraction and de-noising in nonlinear
  regression.
\newblock \emph{Neural Computing \& Applications}, 10\penalty0 (3):\penalty0
  231--243, 2001.
\newblock
  \newline{\Small\url{https://www.researchgate.net/profile/Leonard-Trejo/publication/243134486_Kernel_PCA_for_Feature_Extraction_and_De-Noising_in_Nonlinear_Regression/links/583f5da508ae8e63e6182cbf/Kernel-PCA-for-Feature-Extraction-and-De-Noising-in-Nonlinear-Regression.pdf}}.

\bibitem[Roweis and Saul(2000)]{roweis2000nonlinear}
S.~T. Roweis and L.~K. Saul.
\newblock Nonlinear dimensionality reduction by locally linear embedding.
\newblock \emph{Science}, 290\penalty0 (5500):\penalty0 2323--2326, 2000.
\newblock
  \newline{\Small\url{https://science.sciencemag.org/content/290/5500/2323.full}}.

\bibitem[Rudi et~al.(2015)Rudi, Camoriano, and Rosasco]{rudi2015less}
A.~Rudi, R.~Camoriano, and L.~Rosasco.
\newblock Less is more: Nystr{\"o}m computational regularization.
\newblock In \emph{Advances in Neural Information Processing Systems}, pages
  1657--1665, 2015.
\newblock \newline{\Small\url{https://arxiv.org/pdf/1507.04717.pdf}}.

\bibitem[Rudi et~al.(2017)Rudi, Carratino, and Rosasco]{rudi2017falkon}
A.~Rudi, L.~Carratino, and L.~Rosasco.
\newblock Falkon: an optimal large scale kernel method.
\newblock In \emph{Advances in Neural Information Processing Systems}, pages
  3888--3898, 2017.
\newblock
  \newline{\Small\url{http://papers.nips.cc/paper/6978-falkon-an-optimal-large-scale-kernel-method.pdf}}.

\bibitem[Sch{\"o}lkopf et~al.(1998)Sch{\"o}lkopf, Smola, and
  M{\"u}ller]{scholkopf1998nonlinear}
B.~Sch{\"o}lkopf, A.~Smola, and K.-R. M{\"u}ller.
\newblock Nonlinear component analysis as a kernel eigenvalue problem.
\newblock \emph{Neural Computation}, 10\penalty0 (5):\penalty0 1299--1319,
  1998.
\newblock
  \newline{\Small\url{https://www.mitpressjournals.org/doi/pdfplus/10.1162/089976698300017467}}.

\bibitem[Sen et~al.(2010)Sen, Singer, and De~Lima]{sen2010finite}
P.~K. Sen, J.~M. Singer, and A.~C.~P. De~Lima.
\newblock \emph{From finite sample to asymptotic methods in statistics},
  volume~28 of \emph{Cambridge series in statistical and probabilistic
  mathematics}.
\newblock Cambridge University Press, 2010.
\newblock
  \newline{\Small\url{https://www.cambridge.org/core/books/from-finite-sample-to-asymptotic-methods-in-statistics/07DFA2860E18EDB1A9FE1FF3B4E07F0C}}.

\bibitem[Shawe-Taylor et~al.(2002)Shawe-Taylor, Williams, Cristianini, and
  Kandola]{shawe2002eigenspectrum}
J.~Shawe-Taylor, C.~K. Williams, N.~Cristianini, and J.~Kandola.
\newblock On the eigenspectrum of the {Gram} matrix and its relationship to the
  operator eigenspectrum.
\newblock In \emph{International Conference on Algorithmic Learning Theory},
  pages 23--40. Springer Science \& Business Media, 2002.
\newblock
  \newline{\Small\url{https://link.springer.com/chapter/10.1007/3-540-36169-3_4}}.

\bibitem[Shawe-Taylor et~al.(2005)Shawe-Taylor, Williams, Cristianini, and
  Kandola]{shawe2005eigenspectrum}
J.~Shawe-Taylor, C.~K. Williams, N.~Cristianini, and J.~Kandola.
\newblock On the eigenspectrum of the {Gram} matrix and the generalization
  error of kernel {PCA}.
\newblock \emph{IEEE Transactions on Information Theory}, 51\penalty0
  (7):\penalty0 2510--2522, 2005.
\newblock
  \newline{\Small\url{https://homepages.inf.ed.ac.uk/ckiw/postscript/gram.pdf}}.

\bibitem[Singh et~al.(2019)Singh, Sahani, and Gretton]{singh2019kernel}
R.~Singh, M.~Sahani, and A.~Gretton.
\newblock Kernel instrumental variable regression.
\newblock In \emph{Advances in Neural Information Processing Systems}, pages
  4593--4605, 2019.
\newblock
  \newline{\Small\url{http://papers.neurips.cc/paper/8708-kernel-instrumental-variable-regression.pdf}}.

\bibitem[Sipser(2013)]{sipser2012introduction}
M.~Sipser.
\newblock \emph{Introduction to the theory of computation}.
\newblock Cengage Learning, 3rd edition, 2013.
\newblock
  \newline{\Small\url{https://www.cengagebrain.co.uk/shop/isbn/9780357670583}}.

\bibitem[Sriperumbudur and Sterge(2017)]{sriperumbudur2017approximate}
B.~Sriperumbudur and N.~Sterge.
\newblock Approximate kernel {PCA} using random features: computational vs.
  statistical trade-off.
\newblock \emph{arXiv preprint arXiv:1706.06296}, 2017.
\newblock \newline{\Small\url{https://arxiv.org/pdf/1706.06296.pdf}}.

\bibitem[Sterge and Sriperumbudur(2021)]{sterge2021statistical}
N.~Sterge and B.~Sriperumbudur.
\newblock Statistical optimality and computational efficiency of {N}yström
  kernel {PCA}.
\newblock \emph{arXiv preprint arXiv:2105.08875v1}, 2021.
\newblock \newline{\Small\url{https://arxiv.org/pdf/2105.08875v1}}.

\bibitem[Sterge et~al.(2020)Sterge, Sriperumbudur, Rosasco, and
  Rudi]{sterge2020gain}
N.~Sterge, B.~Sriperumbudur, L.~Rosasco, and A.~Rudi.
\newblock Gain with no pain: efficiency of kernel-{PCA} by {N}ystr{\"o}m
  sampling.
\newblock In \emph{International Conference on Artificial Intelligence and
  Statistics}, pages 3642--3652. PMLR, 2020.
\newblock
  \newline{\Small\url{http://proceedings.mlr.press/v108/sterge20a/sterge20a.pdf}}.

\bibitem[Vishwanathan et~al.(2010)Vishwanathan, Schraudolph, Kondor, and
  Borgwardt]{vishwanathan2010graph}
S.~V.~N. Vishwanathan, N.~N. Schraudolph, R.~Kondor, and K.~M. Borgwardt.
\newblock Graph kernels.
\newblock \emph{The Journal of Machine Learning Research}, 11:\penalty0
  1201--1242, 2010.
\newblock
  \newline{\Small\url{http://www.jmlr.org/papers/volume11/vishwanathan10a/vishwanathan10a.pdf}}.

\bibitem[Wang et~al.(2016)Wang, Berthet, and Samworth]{wang2016statistical}
T.~Wang, Q.~Berthet, and R.~J. Samworth.
\newblock Statistical and computational trade-offs in estimation of sparse
  principal components.
\newblock \emph{The Annals of Statistics}, 44\penalty0 (5):\penalty0
  1896--1930, 2016.
\newblock
  \newline{\Small\url{https://projecteuclid.org/download/pdfview_1/euclid.aos/1473685263}}.

\bibitem[Wibowo and Yamamoto(2012)]{wibowo2012note}
A.~Wibowo and Y.~Yamamoto.
\newblock A note on kernel principal component regression.
\newblock \emph{Computational Mathematics and Modeling}, 23\penalty0
  (3):\penalty0 350--367, 2012.
\newblock
  \newline{\Small\url{https://link.springer.com/content/pdf/10.1007/s10598-012-9143-0.pdf}}.

\bibitem[Wild et~al.(2021)Wild, Kanagawa, and Sejdinovic]{wild2021connections}
V.~Wild, M.~Kanagawa, and D.~Sejdinovic.
\newblock Connections and equivalences between the {N}ystr{\"o}m method and
  sparse variational {G}aussian processes.
\newblock \emph{arXiv preprint arXiv:2106.01121}, 2021.
\newblock \newline{\Small\url{https://arxiv.org/pdf/2106.01121.pdf}}.

\bibitem[Williams(2002)]{williams2002connection}
C.~K. Williams.
\newblock On a connection between kernel {PCA} and metric multidimensional
  scaling.
\newblock \emph{Machine Learning}, 46\penalty0 (1):\penalty0 11--19, 2002.
\newblock
  \newline{\Small\url{https://link.springer.com/content/pdf/10.1023/A:1012485807823.pdf}}.

\bibitem[Williams and Seeger(2001)]{williams2001using}
C.~K. Williams and M.~Seeger.
\newblock Using the {Nystr{\"o}m} method to speed up kernel machines.
\newblock In \emph{Advances in Neural Information Processing Systems}, pages
  682--688, 2001.
\newblock
  \newline{\Small\url{http://papers.nips.cc/paper/1866-using-the-nystrom-method-to-speed-up-kernel-machines.pdf}}.

\bibitem[Wold et~al.(1987)Wold, Esbensen, and Geladi]{wold1987principal}
S.~Wold, K.~Esbensen, and P.~Geladi.
\newblock Principal component analysis.
\newblock \emph{Chemometrics and Intelligent Laboratory Systems}, 2\penalty0
  (1-3):\penalty0 37--52, 1987.
\newblock
  \newline{\Small\url{https://www.sciencedirect.com/science/article/abs/pii/0169743987800849}}.

\bibitem[Yu et~al.(2015)Yu, Wang, and Samworth]{yu2015useful}
Y.~Yu, T.~Wang, and R.~J. Samworth.
\newblock A useful variant of the {D}avis--{K}ahan theorem for statisticians.
\newblock \emph{Biometrika}, 102\penalty0 (2):\penalty0 315--323, 2015.
\newblock \newline{\Small\url{https://arxiv.org/pdf/1405.0680.pdf}}.

\bibitem[Zhang et~al.(2016)Zhang, Yang, Yi, Jin, and Zhou]{zhang2016stochastic}
L.~Zhang, T.~Yang, J.~Yi, R.~Jin, and Z.-H. Zhou.
\newblock Stochastic optimization for kernel {PCA}.
\newblock In \emph{Proceedings of the Thirtieth AAAI Conference on Artificial
  Intelligence}, 2016.
\newblock
  \newline{\Small\url{https://www.aaai.org/ocs/index.php/AAAI/AAAI16/paper/viewFile/12072/11878}}.

\bibitem[Zwald and Blanchard(2005)]{zwald2005convergence}
L.~Zwald and G.~Blanchard.
\newblock On the convergence of eigenspaces in kernel principal component
  analysis.
\newblock In \emph{Advances in Neural Information Processing Systems}, 2005.
\newblock
  \newline{\Small\url{https://hal.archives-ouvertes.fr/file/index/docid/373803/filename/Nips2005mod.pdf}}.

\end{thebibliography}
\bibliographystyle{abbrvnat}

\end{document}